\documentclass[journal]{IEEEtran}
\usepackage{amssymb}
\usepackage{amsmath}
\usepackage{dcolumn}
\usepackage{bm}
\usepackage{color}
\usepackage{graphicx}
\usepackage{subfigure} 
\usepackage{tabularx}
\usepackage{array}
\usepackage{amsthm}
\usepackage{cite}
\usepackage{hyperref}
\usepackage{xcolor}
\usepackage{xpatch}

\newtheorem{theorem}{Theorem}
\newtheorem{corollary}{Corollary}
\newtheorem{lemma}{Lemma}
\newtheorem{remark}{Remark}
\newtheorem{definition}{Definition}
\newtheorem{assumption}{Assumption}
\newtheorem{proposition}{Proposition}

\renewcommand{\t}{^{\mbox{\tiny\sf T}}}
\newcommand{\bremark}{\begin{remark}
\begin{rm}}
\newcommand{\eremark}{ \end{rm}\hfill \rule{1mm}{2mm}
\end{remark} }
\newcommand{\btheorem}{\begin{theorem} \begin{it}}
\newcommand{\etheorem}{\end{it} \hfill \rule{1mm}{2mm}
\end{theorem} }
\newcommand{\blemma}{\begin{lemma} \begin{it} }
\newcommand{\elemma}{ \end{it} \hfill\rule{1mm}{2mm}
\end{lemma} }
\newcommand{\bcorollary}{\begin{corollary} \begin{it} }
\newcommand{\ecorollary}{ \end{it} \hfill\rule{1mm}{2mm}
\end{corollary} }
\newcommand{\bdefinition}{\begin{definition} }
\newcommand{\edefinition}{ \hfill\rule{1mm}{2mm}
\end{definition} }
\newcommand{\bproposition}{\begin{proposition} }
\newcommand{\eproposition}{\hfill \rule{1mm}{2mm}
\end{proposition} }
\newcommand{\bexample}{\begin{example} \begin{rm}}
\newcommand{\eexample}{ \end{rm} \hfill\rule{1mm}{2mm}
\end{example} }

\newcommand{\bassumption}{\begin{assumption} }
\newcommand{\eassumption}{\hfill \rule{1mm}{2mm}
\end{assumption} }

\newcommand{\balgorithm}{\medskip\begin{algorithm} \rm}
\newcommand{\ealgorithm}{ \hfill \rule{1mm}{2mm}\medskip
\end{algorithm} }

\newcommand{\basm}{\begin{assumption} \begin{rm} }
\newcommand{\easm}{ \end{rm} \hfill\rule{1mm}{2mm}
\end{assumption} }

\begin{document}

\title{Spontaneous-Ordering Platoon Control for Multi- Robot Path Navigation Using Guiding Vector Fields}

\author{
\IEEEauthorblockN{Bin-Bin Hu,
Hai-Tao Zhang\IEEEauthorrefmark{1},
Weijia Yao, 
Jianing Ding and
Ming Cao\IEEEauthorrefmark{1}
}
}

\maketitle

\begin{abstract}
In this paper, we propose a distributed guiding-vector-field (DGVF) algorithm for a team of robots to form a {\it  spontaneous-ordering} platoon moving along
a predefined desired path in the $n$-dimensional Euclidean space. Particularly, by adding a path parameter as an additional virtual coordinate to each robot, the DGVF algorithm can eliminate the {\it singular points} where the vector fields vanish, and govern robots to approach a {\it closed} and even {\it self-intersecting} desired path.
Then, the interactions among neighboring robots and a virtual target robot through their virtual coordinates enable the realization of the desired platoon; in particular, relative parametric displacements can be achieved with arbitrary ordering sequences.
Rigorous analysis is provided to guarantee the global convergence to the {\it spontaneous-ordering} platoon on the common desired path from any initial positions. 2D experiments using three HUSTER-0.3 unmanned surface vessels (USVs) are conducted to validate the practical effectiveness of the proposed DGVF algorithm, and 3D numerical simulations are presented to demonstrate its effectiveness and robustness when tackling higher-dimensional multi-robot path-navigation missions and some robots breakdown.
\end{abstract}

\begin{IEEEkeywords}
Swarms, path planning for multiple mobile robots or agents, multi-robot systems, guiding vector fields
\end{IEEEkeywords}

\IEEEpeerreviewmaketitle
\section{Introduction}
Over the years, multi-robot path navigation has attracted increasing attention due to the rich applications in 
searching and rescue, monitoring and reconnaissance, and convey and escort~\cite{macwan2014multirobot,dunbabin2012robots,hu2021distributed2,alonso2015multi,hu2023cooperative}. 
In such a navigation problem, robots are generally governed by two terms: path-following control and multi-robot motion coordination. The former is to 
guide robots to accurately follow some desired paths, which can be achieved by projection-point \cite{samson1995control,aguiar2007trajectory}, line-of-sight (LOS) \cite{fossen2003line,rysdyk2006unmanned} and guiding-vector-field (GVF) methods \cite{kapitanyuk2017guiding,yao2020path}. The latter is to coordinate motions of robots subject to some geometric  constraints. In simple missions within open environments, these coordination constraints can be satisfied by  prescribing some fixed spatial orderings and distributions of robots, which refers to fixed-ordering coordination \cite{yao2019distributed}.

Among the works of multi-robot path navigation, those using fixed-ordering design have been widely explored in the literature.
As pioneering works, an
adaptive controller was developed in \cite{burger2009straight} to 
follow a desired straight-line path. A virtual structure was proposed in \cite{ghommam2010formation} to follow some sinusoidal paths.  
A pragmatic distributed protocol \cite{liu2020scanning} was designed to collectively follow some fitting curved paths. However, these works \cite{burger2009straight,ghommam2010formation,liu2020scanning} were restricted to simple open paths.
Later, it was extended to circles \cite{hu2021distributed1,hu2021bearing} and some other 2D closed curves \cite{zhang2007coordinated,nakai2013vector,de2017circular,doosthoseini2015coordinated}. 
For even more complex 3D paths, an output-regulation-based controller \cite{sabattini2015implementation} was developed to achieve multi-robot path navigation with periodic-changing closed paths in the 3D Euclidean space. Another work~\cite{pimenta2013decentralized} has utilized GVF to follow 3D specific-form paths. 
However, the aforementioned methods in \cite{burger2009straight,ghommam2010formation,liu2020scanning,doosthoseini2015coordinated,hu2021distributed1,hu2021bearing,zhang2007coordinated,nakai2013vector,de2017circular,sabattini2015implementation,pimenta2013decentralized} cannot cope with the desired paths containing {\it self-intersecting} points, which motivates a singularity-free GVF with an additional virtual coordinate in~\cite{yao2021distributed}. Therein, {\it self-intersecting} desired paths were transformed to nonself-intersecting ones in a 
higher-dimensional Euclidean space and then multi-robot path navigation was coordinated with the guaranteed global convergence. Later, such singularity-free GVF was extended for surface navigation with two additional virtual coordinates~\cite{yao2022guiding}.

Still, for more complicated missions in dynamic environments, the previous fixed-ordering design methodology is not ideal, which motivates a more efficient approach to achieve coordination with arbitrary spatial orderings, namely, {\it spontaneous-ordering} coordination to improve efficiency~\cite{sakurama2020multi}. Notably, {\it spontaneous-ordering} coordination does not predetermine the steady-state order of the robots, which implies that the order of the robots does not matter in the multi-robot coordination but only depends on the initial condition of robots. For instance, for maintenance tasks in narrow pipelines, robots must form a platoon  as quickly as possible according to proximity, which then leads to the arbitrary orderings. 
Note that, such {\it spontaneous-ordering} coordination may induce time-varying interaction topologies among robots, which will affect the performance of multi-robot path navigation.
In this pursuit, a distributed hybrid control law was developed in~\cite{lan2011synthesis} to
coordinate the robots to keep a constant parametric separation along the navigation paths. A multifunctional controller was proposed in \cite{reyes2014flocking} integrating flocking, formation regulation, and path following simultaneously. Although these two studies \cite{lan2011synthesis,reyes2014flocking} have tried to address the {\it spontaneous-ordering} coordination scenario, they only considered 
nonself-intersecting paths with local convergence in the 2D plane. The {\it  spontaneous-ordering} multi-robot path navigation with more challenging self-intersected paths and guaranteed global convergence still remain an open problem.

For the specific multi-robot platoon navigation task, a number of existing works also studied the string stability, which is closely related to the attenuation of external disturbances along the platoon \cite{ploeg2013controller}. The early works focused on the string stability for linear robots with a fixed communication topology \cite{ploeg2013lp,besselink2017string,dunbar2011distributed}.
 Later, it was extended to the vehicle platoon with nonlinear dynamics~\cite{monteil2019string,hu2020cooperative,mokogwu2022energy}, switching and uncertain topologies~\cite{xu2022stochastic,feng2022robust}, and even time delays~\cite{liu2020internal,qin2019experimental}. However, string stability in these works \cite{ploeg2013controller,ploeg2013lp,besselink2017string,dunbar2011distributed,xu2022stochastic,feng2022robust,liu2020internal,qin2019experimental,monteil2019string,mokogwu2022energy,hu2020cooperative} 
requires the robots to maneuver with fixed predecessor and follower neighbors (i.e., a fixed-ordering platoon), and restricts in most cases the movement of the platoon only in the 1D Euclidean space. Accordingly, it becomes an urgent yet challenging mission to design a {\it spontaneous-ordering} platoon method in higher-dimensional Euclidean space.

Inspired by the singularity-free GVF reported in \cite{yao2021distributed}, we design a distributed guiding-vector-field (DGVF) algorithm to govern a team of an arbitrary number of robots to form
a {\it spontaneous-ordering} platoon moving along a predefined desired path in the $n$-dimensional Euclidean space.  
Particularly, by adding a path parameter as an additional virtual coordinate to each robot, the DGVF algorithm can eliminate the singular points where the vector fields vanish, and govern robots to approach a {\it closed} and even {\it self-intersecting} desired path. Then, the interactions among neighboring robots and a virtual target robot through virtual coordinates lead to the realization of the desired platoon with an arbitrary ordering. The main contribution is summarized as follows.

\begin{enumerate}
 \item We propose a DGVF algorithm to enable robots to approach and maneuver along a {\it closed} and even {\it self-intersecting} desired path while keeping a platoon with an arbitrary ordering simultaneously.
 
 \item  We guarantee the global convergence to the {\it spontaneous-ordering} platoon on the desired path from any initial positions, and reduce communication and computation costs by transmitting only virtual coordinates among neighboring robots.

\item We establish a multi-USV navigation system and conduct 2D experiments with three HUSTER-0.3 USVs to demonstrate the practical effectiveness of the proposed DVGF algorithm. Moreover, we perform 3D numerical simulations to show its effectiveness and robustness when tackling higher-dimensional navigation missions and some robots breakdown.

\end{enumerate}

The technical novelty of this paper is three-fold. First of all, different from the previous GVF \cite{zhang2007coordinated,nakai2013vector,de2017circular,doosthoseini2015coordinated,sabattini2015implementation,pimenta2013decentralized,yao2021distributed,yao2022guiding} focusing on the fixed-ordering multi-robot path navigation, the present paper designs a DGVF algorithm by utilizing the time-varying interactions among neighboring robots and a virtual target robot through their virtual coordinates 
to address a more challenging {\it spontaneous-ordering} multi-robot path navigation problem. Secondly, the present paper guarantees the global convergence to the {\it  spontaneous-ordering} platoon in presence of strongly nonlinear couplings induced by the ordering flexibility. Thirdly, experiments with three HUSTER-0.3 USVs in a multi-USV navigation system are conducted to demonstrate the practical effectiveness of the proposed DGVF algorithm. Still worth mentioning is that, by using time-varying neighboring interactions, the present DGVF algorithm can even tackle the case when some robots breakdown whereas the previous GVF approaches \cite{zhang2007coordinated,nakai2013vector,de2017circular,doosthoseini2015coordinated,sabattini2015implementation,pimenta2013decentralized,yao2021distributed,yao2022guiding} do not work in such cases.

The remainder of this paper is organized as follows. Section~II introduces preliminaries and the formulation of the problem. The main technical results are elaborated in Section III. 2D experiments using USVs and 3D numerical simulations are both conducted in Section IV. Finally, conclusions are drawn in Section V. 

Throughout the paper, the real numbers and positive real numbers are denoted by $\mathbb{R},\mathbb{R}^+$, respectively. The $n$-dimensional Euclidean space is denoted by $\mathbb{R}^n$. The integer numbers are denoted by $\mathbb{Z}$. The notation $\mathbb{Z}_i^j$ represents the set $\{m\in \mathbb{Z}~|~i\leq m\leq j\}$. The Kronecker product is denoted by~$\otimes$. The $n$-dimensional identity matrix is represented by~$I_n$. The $N$-dimensional column vector consisting of all 1's is denoted by $\mathbf{1}_N$.

\section{Preliminaries}

\subsection{Higher-Dimensional GVF}
Suppose a desired path $\mathcal P$ in the $n$-dimension Euclidean space is characterized by the zero-level set
of the implicit functions $\phi(\sigma)$ \cite{seron1999feedback,yao2020vector},
\begin{align}
\label{inplicit_function}
\mathcal P:=\{ \sigma\in\mathbb{R}^n~|~\phi(\sigma)=0\},
\end{align}
where $\sigma\in\mathbb{R}^n$ are the coordinates and $\phi(\cdot): \mathbb{R}^n\rightarrow\mathbb{R}$ is twice continuously differentiable, i.e., $\phi(\cdot)\in\mathcal C^2$.
Unlike conventional methods \cite{samson1995control,aguiar2007trajectory} to measure the error between a point $p_0\in\mathbb{R}^n$ and the desired path $\mathcal P$ by $\mbox{dist}(p_0,\mathcal P)=\mbox{inf}\{\|p-p_0\|~|~p\in\mathcal P\}$, the implicit functions $\phi(\sigma)$ provide a more convenient way to measure the path-following errors with $\phi(p_0)$. 
However, there may exist some pathological situations, i.e., settling down of \big($\|\phi(p_0(t))\|$  to zero along the trajectory $p_0(t)$ does not necessarily imply that $\mbox{dist}(p_0(t),\mathcal P)$ converges to~$0$ as $t \to \infty$, see \cite{el2007passivity, yao2018robotic}\big), which can be excluded by the following assumption.

\begin{assumption}
\label{assp_error}
\cite{kapitanyuk2017guiding} For any given $\kappa>0$ and a point 
$p_0(t)$, one has that
$\inf\{\|\phi(p_0)\|: \mbox{dist}(p_0,\mathcal P)\geq \kappa\}>0.$
\end{assumption}

Assumption~\ref{assp_error} guarantees that the path-following errors $\|\phi(p_0)\|$ are utilized to measure “how close" a point $p_0$ is to the desired path $\mathcal P$, i.e., $\lim_{t\rightarrow\infty}\|\phi(p_0(t))\|=0\Rightarrow  \lim_{t\rightarrow\infty}\mbox{dist}(p_0(t),\mathcal P)=0$, which can be satisfied by using some
polynomial or trigonometric functions (see, e.g., \cite{rezende2018robust,goncalves2010vector,yao2018robotic}). 

\begin{figure}[!htb]
\centering
\includegraphics[width=\hsize]{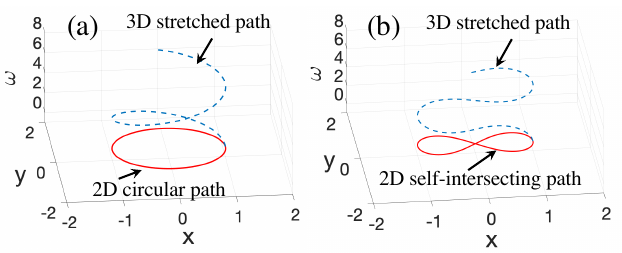}
\caption{ (a) The red solid line is the desired 2D circular path $\mathcal P^{phy}:=\{[\sigma_1, \sigma_2]\t\in\mathbb{R}^2~|~\sigma_1=\cos\omega, \sigma_2=\sin\omega, \omega\in\mathbb{R}\}$, whereas the blue dashed line is the corresponding “stretched”  desired
3D path $\mathcal P^{hgh}:=\{[\sigma_1, \sigma_2, \omega]\t\in\mathbb{R}^3~|~\sigma_1=\cos\omega, \sigma_2=\sin\omega\}$. (b) The red solid line is the desired 2D {\it self-intersecting} Lissajous path $\mathcal P^{phy}:=\{[\sigma_1, \sigma_2]\t\in\mathbb{R}^2~|~\sigma_1=\cos\omega/(1+0.5(\sin\omega)^2), \sigma_2=\cos\omega\sin\omega/(1+0.5(\sin\omega)^2), \omega\in\mathbb{R}\}$, whereas the blue dashed line is the corresponding “stretched” desired 3D path $\mathcal P^{hgh}:=\{[\sigma_1, \sigma_2, \omega]\t\in\mathbb{R}^2~|~\sigma_1=\cos\omega/(1+0.5(\sin\omega)^2), \sigma_2=\cos\omega\sin\omega/(1+0.5(\sin\omega)^2)\}$.}
\label{illustration_HGVF}
\end{figure}
Using the characterization of the desired path $\mathcal P$ in~\eqref{inplicit_function}, we are ready to introduce a higher-dimensional GVF to address the single-robot path navigation problem.

\begin{definition}
\label{definition_GVF}
(Higher-dimensional GVF)~\cite{yao2021singularity} Given the desired path $\mathcal P^{phy}$ in the $n$-dimension Euclidean space satisfying Assumption~\ref{assp_error} and parameterized by 
$$
\mathcal P^{phy}:=\{[\sigma_1, \cdots, \sigma_n]\t\in\mathbb{R}^n~|~\sigma_j=f_j(\omega), j\in\mathbb{Z}_1^n, \omega\in\mathbb{R}\}
$$ 
with the $j$-th cooridinate $\sigma_j$, the path parameter $\omega$, and the function $f_j \in \mathcal C^2$, there exists a corresponding desired path $\mathcal P^{hgh}$ in the higher-dimensional  Euclidean space
$$
\mathcal P^{hgh}:=\{\xi\in\mathbb{R}^{n+1}~|~ \phi_j(\xi)=0, j\in\mathbb{Z}_1^n\},
$$
where $\xi:=[\sigma_1,\dots, \sigma_n, \omega]\t$ 
are the generalized coordinates by regarding $\omega$ as an additional coordinate, and $\phi_{j}(\xi):=\sigma_j-f_j(\omega), j\in\mathbb{Z}_1^n$ 
are the implicit functions to measure the path-following errors. Since 
$\mathcal P^{phy}$ 
corresponds to the projection of 
$\mathcal P^{hgh}$ 
spanned on the first $n$ coordinates, a higher-dimensional GVF $\chi^{hgh}\in\mathbb{R}^{n+1}$ can be designed as follows,
\begin{align}
\label{eq_GVF}
\chi^{hgh}=\times (\nabla\phi_1, \cdots, \nabla\phi_n)-\sum_{j=1}^nk_{j}\phi_j\nabla\phi_j,
\end{align}
which can govern a robot to approach and maneuver along the desired path $\mathcal P^{phy}$ by projecting $\chi^{hgh}$ to the first $n$-dimensional Euclidean space. Here, $k_j\in\mathbb{R}^+$ is the gain, $\nabla\phi_j(\cdot): \mathbb{R}^{n+1}\rightarrow\mathbb{R}^{n+1}$ denotes the gradient of $\phi_j$ w.r.t. $\xi_j$ and $\times(\cdot)$ represents the wedge product~\cite{galbis2012vector}.  
\end{definition}

The higher-dimensional GVF $\chi^{hgh}$ in~\eqref{eq_GVF} is capable of providing a propagation direction 
along the desired path $\mathcal P^{phy}$ with the first term 
$\times (\nabla\phi_1, \cdots, \nabla\phi_n)\in\mathbb{R}^{n+1}$ 
orthogonal to all the gradients $\nabla\phi_j, j\in\mathbb{Z}_1^n$, and approaching the desired path $\mathcal P^{phy}$
with the second term of 
$\sum_{j=1}^nk_{j}\phi_j\nabla\phi_j$. 
In \cite{yao2021singularity}, it has been shown that the higher-dimensional GVF $\chi^{hgh}$ can eliminate 
the {\it singular points} (i.e., $\chi^{hgh}=0$) by adding the virtual coordinate $\omega,$ and hence guarantee the global convergence 
to even {\it self-intersecting} desired paths.

\begin{remark}
By transforming the path parameter~$\omega$ into an additional virtual coordinate, the desired closed and self-intersecting paths $\mathcal P^{phy}\in\mathbb{R}^n$ in $\mathbb{S}^1$ are “cut” and “stretched” into the higher-dimensional desired paths $\mathcal P^{hgh}\in\mathbb{R}^{n+1}$, and become unbounded and nonself-intersecting after introducing the additional dimension $\omega$ \cite{yao2021singularity}. Examples of such a “stretching” operation are illustrated in Fig.~\ref{illustration_HGVF},  where the desired 2D circular and self-intersecting paths $\mathcal P^{phy}$ have been transformed into the corresponding unbounded desired 3D paths $\mathcal P^{hgh}$, respectively. Moreover, the higher-dimensional GVF $\chi^{hgh}\in\mathbb{R}^{n+1}$ in Eq.~\eqref{eq_GVF} is designed for the “stretched” higher-dimensional desired paths $\mathcal P^{hgh}\in\mathbb{R}^{n+1}$, where $\chi^{hgh}$ is then projected into its first $n$ coordinates to govern the robot to approach and move along the original desired paths $\mathcal P^{phy}\in\mathbb{R}^n$.
\end{remark}

\subsection{Multi-Robot Path Navigation}
We consider a multi-robot system consisting of $N$ robots denoted by ${\cal V}=\{1,2,\dots, N\}$. Each one is described by the single integrator kinematics, 
\begin{align}
\label{kinetic_F}
 \dot{x}_i &=u_i + d_i, i\in\mathcal V,
\end{align}
where $x_i(t) :=[x_{i,1}, \dots, x_{i,n}]\t\in\mathbb{R}^n$ represent 
the positions and $u_i(t):=[u_{i,1}, \dots, u_{i,n}]\t\in\mathbb{R}^n$ 
the control inputs of the robot $i$, $d_i:=[d_{i,1}, \dots, d_{i,n}]\t\in\mathbb{R}^n$ the external disturbances, such as the state estimation errors, feedback-linearization errors, wind, and currents.
Note that the inputs $u_i$ in Eq.~\eqref{kinetic_F} can be regarded as the desired high-level guidance velocities when applied to practical robots with higher-order dynamics, which are thus applicable to various robots with the hierarchical control structure, such as unmanned aerial vehicles (UAVs), and unmanned surface vessels (USVs) \cite{hu2021bearing,rezende2018robust,yao2021singularity}.

Suppose the $i$-th desired path $\mathcal P_i^{phy}$ for robot $i, i\in \mathcal V,$ in the $n$-dimensional Euclidean space is described by, 
\begin{align}
\label{desired_path}
\mathcal P_i^{phy}:=&\{\sigma_i:=[\sigma_{i,1} , \dots,           
                    \sigma_{i,n}]\t\in\mathbb{R}^n~|~\nonumber\\
                   &\sigma_{i,j}=f_{i,j}(\omega_i), j\in\mathbb{Z}_1^n, \omega_i\in \mathbb{R}\},
\end{align}
where $\sigma_i$ are the coordinates of the desired path 
$\mathcal P_i^{phy}$, 
$f_{i,j}(\omega_i)\in\mathcal C^2, j\in\mathbb{Z}_1^n$ and $\omega_i$ 
are the parametric functions and the virtual coordinate of robot $i$, respectively. Here, $f_{i,j}(\omega_i)$ in Eq.~\eqref{desired_path} are in the same parametric form $f_{i,j}(\cdot)$ for all the robots~$\mathcal V$ but with different virtual coordinates $\omega_i, i\in\mathcal V$, which then make $\mathcal P_i^{phy}$ in Eq.~\eqref{desired_path} a common desired path for the multi-robot platoon task later.
Then, the sensing neighborhood $\mathcal N_i$ of robot $i$ is defined by
\begin{align}
\label{sening_neighbor}
\mathcal N_i(t):=\{k\in {{\cal V}}, k\neq i~|~|\omega_{i,k}(t)| < R\}
\end{align} 
with the sensing radius 
$R\in(r, \infty)$, the safe radius $r$ and $\omega_{i,k}:=\omega_i-\omega_k$. 
Since the relative parametric value $|\omega_{i,k}(t)|$ is time-varying, one has that $\mathcal N_i$ is time-varying as well, which can lead to a {\it spontaneous-ordering} platoon later whereas posing challenging issues in the stability analysis.

\begin{figure}[!htb]
\centering
\includegraphics[width=\hsize]{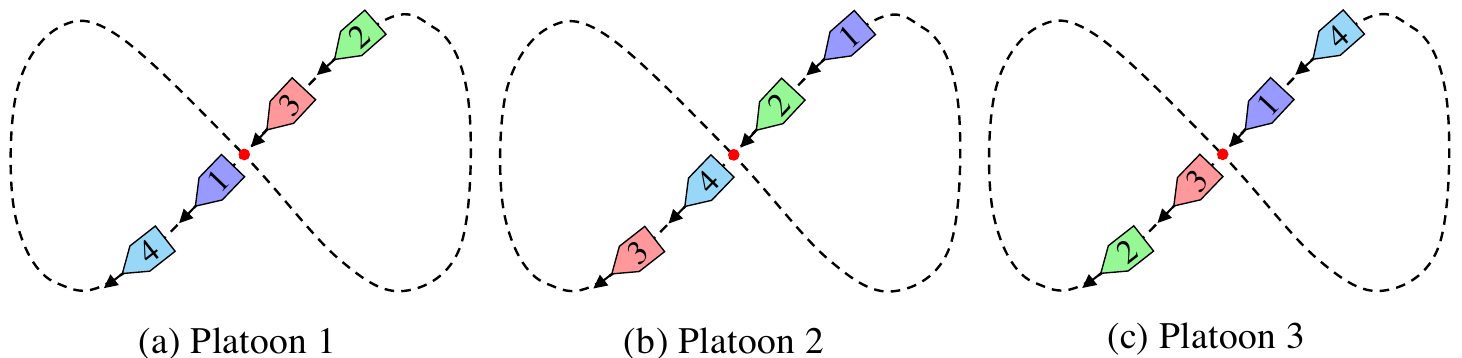}
\caption{An illustrative example where four robots (different colors) form three distinct-ordering platoons whereas moving along a desired 2D {\it self-intersecting} Lissajous path. (The red point denotes the self-interesting point of the path.)}
\label{definittion}
\end{figure}

Note that the common desired path $\mathcal P_i^{phy}$ has been scaled to each robot’s virtual coordinate $\omega_i$, which can stipulate the common scale to determine the neighborhood $\mathcal N_i(t)$ in \eqref{sening_neighbor}. An intuitive example of $\mathcal N_i(t)$ is 
that when $\mathcal P_i^{phy}$ in~\eqref{desired_path} is a line e.g., $f_{i,1}=\omega_i, f_{i,j}=0, j\in \mathbb{Z}_2^n, i\in\mathcal V$, 
the relative value $|\omega_{i,k}|$ becomes the 
$x$-axis distance, which implies that the definition of $\mathcal N_i$ in \eqref{sening_neighbor} is hence reasonable and feasible in practice.

Moreover, from Definition~\ref{definition_GVF}, $\mathcal P_i^{phy}$ in \eqref{desired_path}
can be transformed to the corresponding common desired path $\mathcal P_i^{hgh}$ in the higher-dimensional Euclidean space 
\begin{align}
\label{high_desired_path}
\mathcal P_i^{hgh}:=&\{[\sigma_{i,1} , \dots, \sigma_{i,n}, \omega_i]\t\in\mathbb{R}^{n+1}~|\nonumber\\
				&\sigma_{i,j}=f_{i,j}(\omega_i), j\in\mathbb{Z}_1^n\}.
\end{align}
Denoting $p_i:=[x_{i,1}, \dots, x_{i,n}, \omega_i]\t\in\mathbb{R}^{n+1}$ and substituting the positions $x_i=[x_{i,1}, \dots, x_{i,n}]\t$ of robot $i$ into $\mathcal P_i^{hgh}$ in~\eqref{high_desired_path},  the path-following errors $\phi_{i,j}(p_i)\in\mathbb{R}, \forall j\in\mathbb{Z}_1^n,$ between robot $i$ and the desired higher-dimensional path $\mathcal P_i^{hgh}$ are 
\begin{align}
\label{err_phi}
\phi_{i,j}(p_i)=&x_{i,j}-f_{i,j}(\omega_i),  j\in \mathbb{Z}_1^n.
\end{align}
Then, all the robots $\mathcal V$ achieve the desired multi-robot path navigation mission once the path-following errors $\phi_{i,j}(p_i), \forall j\in\mathbb{Z}_1^n,$ converge to zeros, i.e., 
\begin{align*}
\lim_{t\rightarrow\infty} \phi_{i,j}(p_i(t))=0,\forall i\in\mathcal V, j\in \mathbb{Z}_1^n.
\end{align*}

\subsection{Spontaneous-Ordering Platoon}
According to the parametric path $\mathcal P_i^{phy}$ in \eqref{desired_path} and the path-following errors $\phi_{i,j}(p_i)$ in \eqref{err_phi}, we are ready to introduce the {\it spontaneous-ordering} platoon for multi-robot path navigation problem.

\begin{definition}
\label{CPF_definition}
(Spontaneous-ordering platoon)
A group of robots~${\cal V}$ governed by \eqref{kinetic_F} collectively form a {\it spontaneous-ordering} platoon moving along a common desired path $\mathcal P_i^{phy}$~\eqref{desired_path} under Assumption~\ref{assp_error}, if the following claims are fulfilled,
\begin{align}
&1) \lim_{t\rightarrow\infty} \phi_{i,j}(p_i(t))=0,\forall i\in\mathcal V, j\in \mathbb{Z}_1^n, \nonumber\\
&2) \lim_{t\rightarrow\infty} \dot{\omega}_{i}(t)=\lim_{t\rightarrow\infty} \dot{\omega}_{k}(t)\neq 0, \forall i\neq k \in{\cal V}, \nonumber\\
&3)~r<\lim_{t\rightarrow\infty} |\omega_{s[k]}(t)-\omega_{s[k+1]}(t)|<R,  \forall k\in \mathbb{Z}_1^{N-1}, \nonumber\\
&4)~|\omega_{i,k}(t)|>r, \forall t\geq 0, \;  \forall i\neq k \in{\cal V},
\end{align} 
where $\dot{\omega}_i$ denotes the derivative of $\omega_i$, $R\in\mathbb{R}^+, r\in\mathbb{R}^+$ 
are the specified sensing and safe radius in \eqref{sening_neighbor}, respectively. Here, $\omega_{s[1]}<\omega_{s[2]}<\dots<\omega_{s[N]}$ 
are the states of the virtual coordinates with an arbitrary sequence 
$\{s[1], s[2], \dots, s[N]\}$ in an ascending order when $t\rightarrow\infty$.
\end{definition}

In Definition~\ref{CPF_definition}, Claim 1) indicates that all the robots converge to the common desired path $\mathcal P_i^{phy}$. Claim 2) implies that all the robots move along the common desired path and maintain relative parametric displacements $\omega_{i,k}, i\neq k\in \mathcal V$, i.e., the parametric displacement of the platoon is fixed. 
Claim~3) assures the ordering of the platoon is spontaneous with an arbitrary sequence. By properly selection of $R$ and $r$, it is only required that the limiting relative value of adjacent virtual coordinates $|\omega_{s[k]}(t)-\omega_{s[k+1]}(t)|$ can be set in an acceptable region (i.e., $r<|\omega_{s[k]}(t)-\omega_{s[k+1]}(t)|<R$), which is reasonable in practice. Claim 4) avoids the overlapping of virtual coordinates, which thus guarantees inter-robot collision avoidance. From Claims 3) and 4), the ordering flexibility of the platoon indicates that the steady-state order of the robots cannot be stipulated by the virtual coordinates $\omega_i$ in advance, and depends on the initial condition of the robots. It will pose challenges in the platoon analysis by \emph{time-varying} neighbor relations induced by such platoon ordering flexibility; in sharp comparison, the (desired) neighbor relationships in fixed-ordering platoons are usually time-invariant and thus the controls are easier to be designed, and implemented. An example of {\it spontaneous-ordering} platoon is illustrated in Fig.~\ref{definittion}, where the platoons 1, 2, 3 all fulfill the four claims in Definition~\ref{CPF_definition} but with distinct ordering sequences.

\subsection{Problem Formulation}
Let 
$\partial f_{i,j}(\omega_i):=\partial f_{i,j}(\omega_i)/\partial \omega_i$
be the derivative of 
$f_{i,j}(\omega_i)$ w.r.t.
$\omega_i$, 
one has that the gradient of 
$\phi_{i,j}(p_i)$ in \eqref{err_phi} along $p_i\in\mathbb{R}^{n+1}$ is calculated as follows
\begin{align}
\label{gradient_fi}
\nabla\phi_{i,j}(p_i)&:=[0, \dots, 1, \dots,
-\partial f_{i,j}(\omega_i)]\t\in\mathbb{R}^{n+1},
\end{align}
which implies that the time derivative of 
$\phi_{i,j}(p_i)$ is
\begin{align}
\label{representation_gradient}
\dot{\phi}_{i,j}(p_i)=&\nabla\phi_{i,j}(p_i)\t\dot{p}_i, i\in\mathcal V, j\in \mathbb{Z}_1^n.
\end{align}
Meanwhile, $u_i^{\omega}$ is defined as the desired input for the dynamic of virtual coordinate $\dot{\omega}_i$, i.e., 
\begin{align}
\label{dynamic_omega}
\dot{\omega}_i=u_i^{\omega}.
\end{align}
Let $\dot{\phi}_{i,j}=\dot{\phi}_{i,j}(p_i),\partial f_{i,j}=\partial f_{i,j}(\omega_i), i\in\mathcal V, j\in \mathbb{Z}_1^n$ for conciseness. 
Rewriting  $\Phi_i:=[\phi_{i,1}, \phi_{i,2}, \dots, \phi_{i,n}]\t$, 
$u_i:=[u_{i,1}, u_{i,2}, \dots, u_{i,n}]\t$ and combining Eqs.~\eqref{kinetic_F}, \eqref{representation_gradient} and \eqref{dynamic_omega} together yields
\begin{align}
\label{dynamic_path}
\begin{bmatrix}
\dot{\Phi}_i\\
\dot{\omega}_i
\end{bmatrix}
=
D_i 
\begin{bmatrix}
u_i+d_i\\
u_i^{\omega}
\end{bmatrix}
\end{align}
with
\begin{align*}
D_i=& \begin{bmatrix}                      
                       1 & 0 & \cdots & -\partial f_{i,1}\\
                       0 & 1 &\cdots & -\partial f_{i,2}\\
                       \vdots & \vdots & \ddots &\vdots\\
                       0 & \cdots &1 & -\partial f_{i,n}\\
                       0 & \cdots &0 & 1
                       \end{bmatrix}\in \mathbb{R}^{n+1\times n+1}.
\end{align*}

Now, we are ready to introduce the main problem addressed by this paper.

{\bf Problem 1}: ({\it Spontaneous-ordering} platoon in multi-robot path navigation task)
Design a distributed algorithm
\begin{align}
\label{pro_desired_signal}
\{u_{i}, u_i^\omega\}:=g(\phi_{i,1},\dots,\phi_{i,n}, 
\omega_i, \omega_k), i\in\mathcal V, k\in \mathcal N_i,
\end{align}
for the multi-robot system governed by \eqref{kinetic_F}, \eqref{dynamic_path} and \eqref{pro_desired_signal} to attain the {\it  spontaneous-ordering} platoon, as given in Definition~\ref{CPF_definition}.

\section{Main Technical Results}
Firstly, it follows from Eqs.~\eqref{eq_GVF},~\eqref{err_phi},~\eqref{gradient_fi} that the higher-dimensional GVF 
$\chi_i^{hgh}(x_{i,1},$ $ x_{i,2}, \dots, x_{i,n}, \omega)\in\mathbb{R}^{n+1}$ 
for robot $i$ is (see, e.g., \cite{yao2021singularity}),
\begin{align}
\label{desired_GVF}
\chi_i^{hgh}=&\times (\nabla\phi_{i,1},  \cdots, \nabla\phi_{i, n})-\sum_{j=1}^nk_{i,j}\phi_{i,j}\nabla\phi_{i,j}\nonumber\\							      
            =&\begin{bmatrix}
        	              (-1)^n \partial f_{i,1}-k_{i,1}\phi_{i,1} \\
	              \vdots\\
                      (-1)^n \partial f_{i,n}-k_{i,n}\phi_{i,n}\\
                      (-1)^n+\sum\limits_{j=1}^nk_{i,j}\phi_{i,j}\partial f_{i,j}
        \end{bmatrix}.        
\end{align}
It follows from the $\chi_i^{hgh}$ in \eqref{desired_GVF} that 
the DGVF algorithm for Problem~1 is designed as follows,
\begin{align}
\label{desired_law}
u_{i,j}=&(-1)^n \partial f_{i,j}-k_{i,j}\phi_{i,j} +\widehat{d}_{i,j},~
                \forall j\in \mathbb{Z}_1^n, \nonumber\\
u_i^\omega=&(-1)^n+\sum_{j=1}^nk_{i,j}\phi_{i,j}\partial f_{i,j}-
                c_i(\omega_i-\widehat{\omega}_i)+\eta_i,
\end{align}
where 
$k_{i,j}, c_i\in\mathbb{R}^+, i\in\mathcal V, j\in \mathbb{Z}_1^n$
are the corresponding gains, $\omega_i, \phi_{i,j}$, $\partial f_{i,j}, j\in\mathbb{Z}_1^n$
are given in \eqref{desired_path} and \eqref{dynamic_path}, respectively. $\widehat{d}_i:=[\widehat{d}_{i,1}, \dots, \widehat{d}_{i,n}]\t\in\mathbb{R}^n$ represents an additional well-designed observer to compensate for the external disturbances $d_i$ in Eq.~\eqref{kinetic_F} (refer to Remark~\ref{observer_disturbance} for more details).
$\widehat{\omega}_i$ is defined as the estimation of the target virtual coordinate $\omega^{\ast}$ for robot $i$, where $\omega^{\ast}$ is the corresponding virtual coordinate of a virtual target robot labeled $\ast$ moving on the desired path $\mathcal P_{\ast}^{phy}$ governed by the designed GVF $\chi_{\ast}^{hgh}$ in~\eqref{desired_GVF}. 
Since the virtual target robot $\ast$ is already moving on the common desired path $\mathcal P_{\ast}^{phy}$, one has that $\phi_{\ast,j}=0, \forall j\in\mathbb{Z}_1^n$, which implies that the derivative of $\omega^{\ast}$ satisfies 
\begin{align}
\label{target_velocity}
\dot{\omega}^{\ast}=(-1)^n+\sum_{j=1}^nk_{\ast,j}\phi_{\ast,j}f_{\ast,j}'=(-1)^n
\end{align}
as observed from Eq.~\eqref{desired_GVF}.

Further, $\eta_i$ in \eqref{desired_law} denotes the inter-agent repulsive term which satisfies 
\begin{align}
\label{de_eta}
\eta_i=\sum_{k\in \mathcal N_i} \alpha (|\omega_{i,k}|)     
        \frac{\omega_{i,k} }{|\omega_{i,k}|}
\end{align}
with 
$\omega_{i,k}:=\omega_i-\omega_k$, $\mathcal{N}_i$ given in \eqref{sening_neighbor}, and the continuous function 
$\alpha(s):(r, \infty)\rightarrow [0, \infty)$ 
(see e.g.~\cite{chen2019cooperative})
satisfying 
\begin{align}
\label{alpha}
\alpha(s)=0, \forall s\in[R, \infty),
\lim_{s\rightarrow r^{+}}\alpha(s)=\infty.
\end{align}
An illustrative example of 
$\alpha(s)$ is (see, e.g.~\cite{chen2019cooperative}),
\begin{equation}
\label{potential_alpha}
\alpha(s)=
\left\{
\begin{array}{llr}
\frac{1}{s-r}-\frac{1}{R-r} & r<s\leq R,\\
0 & s>R,
\end{array}
\right.
\end{equation}
where
$\alpha(s)$ 
is monotonically decreasing if 
$s\in(r, R]$ 
and equal~$0$ if 
$s\in(R, \infty)$. It implies that
$\alpha(s)$
is continuous in the domain 
$(r, \infty)$.

Next, we will prove that the multi-robot system 
governed by \eqref{kinetic_F}, \eqref{dynamic_path} and \eqref{desired_law} 
satisfies the property P1.

\begin{itemize}

\item[{\bf P1}:] Robots ${\cal V}$ achieve a {\it spontaneous-ordering} platoon in the multi-robot navigation task.
 
\end{itemize}

To this end, conditions C1-C5 are required.

\begin{itemize}

\item[{\bf C1}:] 
The initial positions and virtual coordinates of the robots satisfy
$\|x_{i}(0)-x_{k}(0)\|>0, |\omega_{i,k}(0)|>r, \; \forall i\neq k \in {\cal V}$.

\item[{\bf C2}:] 
The first and second derivatives of 
$f_{i,j}(\omega_i),i\in\mathcal V, j\in \mathbb{Z}_1^n$
are bounded.

\item[{\bf C3}:] 
The estimation $\widehat{\omega}_i$ 
converges to the target virtual coordinate 
$\omega^{\ast}$ exponentially, i.e., $\lim_{t\rightarrow\infty}\widehat{\omega}_i(t)-\omega^{\ast}(t)=0, 
i\in \mathcal V $, exponentially.

\item[{\bf C4}:] 
The total length 
$\mathcal L_i^{phy}$
of the common desired path $\mathcal P_i^{phy}$ is required to be great than the length of the platoon, i.e.,  
$\mathcal L_i^{phy}>\int_0^{NR}\sqrt{\sum_{j=1}^n \partial f_{i,j}^2(s)}ds$.

\item[{\bf C5}:] 
The external disturbances $d_i$ in \eqref{kinetic_F} and their first-order derivatives $\dot{d}_i$ are all bounded, i.e., $\|d_i\|\leq \beta_{i,1}, \|\dot{d}_i\|\leq \beta_{i,2}, i\in\mathcal V,$  for some positive constants $\beta_{i,1}, \beta_{i,2}\in\mathbb{R}^+$~\cite{peng2020output}.

\end{itemize}

\begin{remark}
Condition C1 is reasonable and necessary, which will be utilized to avoid the overlapping of robots.
Condition C2 is used to prevent the common desired path from changing too fast, 
see, e.g.,~\cite{yao2021distributed}, which is necessary 
for the global convergence analysis later. Condition C4 assures that there exists enough room of the common desired path to accommodate all the robots, 
otherwise the head robot in the platoon may collide with the 
tail robot, which fails to form a satisfactory platoon. 
\end{remark}

\begin{remark}
\label{remark_c3}
Condition C3 is the existence of a distributed estimator for 
the target virtual coordinate 
$\omega^{\ast}$ 
with a constant velocity $\dot{\omega}^{\ast}=(-1)^n$ in \eqref{target_velocity}. Such
a problem has been well studied in the literature, e.g., \cite{olfati2004consensus,hong2006tracking,zhao2013distributed} 
with a connected and undirected topology, and even can be easily achieved by broadcasting $\omega^{\ast}$ with a finite-time technique,
which is out of the main scope of this paper. 
To make the whole design complete, the distributed estimator endowing exponential convergence has the following structure, 
\begin{align}
\dot{\widehat{\omega}}_i
=&
\gamma_1\bigg(\sum_{j\in N_i^c}(\widehat{\omega}_{j}
-
\widehat{\omega}_{i})
+
b_i(\omega^{\ast}-\widehat{\omega}_{i})\bigg)+\widehat{\varsigma}_i,\nonumber\\
\dot{\widehat{\varsigma}}_i=&
\gamma_1\gamma_2\bigg(\sum_{j\in N_i^c}(\widehat{\omega}_{j}
-
\widehat{\omega}_{i})
+
b_i(\omega^{\ast}-\widehat{\omega}_{i})\bigg),
\end{align} 
where $\widehat{\omega}_i, \widehat{\varsigma}_i$ are the $i$-th robot's estimates of $\omega^{*}$ and $\dot{\omega}^{*}$, respectively, $\gamma_1, \gamma_2\in\mathbb{R}^+$ are the estimated gain,  $b_i=1$ 
if robot $i$ has access to $\omega^{\ast}$ and $b_i=0$, otherwise. 
$N_i^c, i\in{\cal V}$, represents the communication neighborhood set of the robot~$i$. Let $L\in\mathbb{R}^{N\times N}$
be the Laplacian matrix according to the neighboring set 
$N_i^c, i\in{\cal V}$ and $B:=\mbox{diag}\{b_1, b_2, \dots, b_n\}\in\mathbb{R}^{N\times N}$, one has that the smallest eigenvalue $\bar{\lambda}$ of the matrix $(L+B)$ satisfies $\bar{\lambda}>0$ with a connected communication topology and at least one robot has access to $\omega^{\ast}$.
Denote $\varrho:=[\chi\t, \zeta\t ]\t\in\mathbb{R}^{2N}$ with
$\chi:=[\widehat{\omega}_{1}, \widehat{\omega}_{2}, \dots, \widehat{\omega}_{N}]\t-\mathbf{1}_N\otimes \omega^{\ast}$
and 
$\zeta:=[\widehat{\varsigma}_{1}, \widehat{\varsigma}_{2}, \dots, \widehat{\varsigma}_{N}]\t-\mathbf{1}_N\otimes \dot{\omega}^{\ast}$,
and one has the closed-loop system is  
$\dot{\varrho}=A\varrho$ with
\begin{align*}
A=\begin{bmatrix}
-\gamma_1(L+H) & I_{N}\\
-\gamma_1\gamma_2(L+H) & 0
\end{bmatrix}.
\end{align*}
According to the conditions $\gamma_1>1/(4\gamma_2(1-\gamma_2^2)\bar{\lambda}), 1>\gamma_2>0$ in~\cite{hong2006tracking}, one has
that $\lim_{t\rightarrow\infty}\varrho(t)=0$ exponentially, which indicates that 
$\lim_{t\rightarrow\infty}\widehat{\omega}_i-\omega^{\ast}=0$ exponentially.
\end{remark}

\begin{remark}
\label{observer_disturbance}
Condition C5 is common in real applications. For generally bounded disturbances, there exist various works focusing on the disturbance observer $\widehat{d}_i$ for the compensation of $d_i$ in Eq.~\eqref{kinetic_F}, such as the extended state observers (ESO) and sliding mode observers (SMO), which can estimate the disturbances in finite time~\cite{peng2020output,gu2022disturbance}, i.e., $\lim_{t\rightarrow T_1}\{\widehat{d}_i(t)-d_i(t)\}=0$
with a constant time $T_1>0$. The design of such disturbance observers is out of the scope of this paper. Instead, we assume that the compensation of $d_i$ is achieved by adding a well-designed observer $\widehat{d}_i$ into the original inputs $u_i$, namely, $u_i\rightarrow u_i+\widehat{d}_i$ in DGVF \eqref{desired_law}, and then analyze the influence of estimated disturbance errors in Lemmas~\ref{lemma_finiteescape}-\ref{lemma_pathconvergence} later. Moreover, for constant disturbances, extensive simulations with no additional disturbance observers are shown in Figs.~\ref{distur_1_Lissajous_12}-\ref{Dis_Lissajous_performance2} to illustrate the quantitative influence of disturbances on the {\it spontaneous-ordering} platoon, which demonstrate that the proposed DGVF~\eqref{desired_law} can still guarantee the {\it spontaneous-ordering} platoon under small constant external disturbances.  
\end{remark}

Since the DGVF algorithm \eqref{desired_law} is not well defined at $\omega_{i,k}=0$ or $\omega_{i,k}=r$ because of $\eta_i$ in \eqref{de_eta}, it may exhibit a finite-time-escape behavior (i.e., $u_i^{\omega}(t)=\infty$) for the closed-loop system~\eqref{dynamic_path}. Therefore, we derive the main results in three steps for readers' convenience. In Step~1, we prevent the finite-time-escape behavior in the closed-loop system~\eqref{dynamic_path} (i.e., $\omega_{i,k}(t)\neq0, \omega_{i,k}(t)\neq r, \forall t>0$ and Claim 4)).
In Step~2, we prove that all the robots converge to and then maneuver along a common desired path (i.e., Claims 1)-2) in Definition~\ref{CPF_definition}). In Step~3, we prove the forming of the {\it spontaneous-ordering} platoon (i.e., Claim~3) in Definition~\ref{CPF_definition}).

\begin{lemma}
\label{lemma_finiteescape}
Under conditions C1, C3 and C5, a multi-robot system governed by \eqref{kinetic_F}, \eqref{desired_law} prevents the finite-time-escape behavior, i.e., $\omega_{i,k}(t)\neq0, \omega_{i,k}(t)\neq r, \forall t>0$.
\end{lemma}

\begin{proof}
See Appendix~\ref{Proof_lemma_1}.
\end{proof}

\begin{lemma}
\label{lemma_pathconvergence}
Under conditions C2 and C3, a multi-robot system governed by \eqref{kinetic_F}, \eqref{desired_law} converges to and then moves along the common desired path $\mathcal P_i^{phy}, i\in\mathcal{V} $ in Eq.~\eqref{desired_path}, i.e., 
$\lim_{t\rightarrow\infty} \phi_{i,j}(p_i(t))=0, \lim_{t\rightarrow\infty} \dot{\omega}_{i}(t)=\lim_{t\rightarrow\infty}  \dot{\omega}_{k}(t)\neq 0, \forall i\neq k \in{\cal V}, j\in \mathbb{Z}_1^n$.
\end{lemma}

\begin{proof}
From the definition of 
$\Omega$ in \eqref{value_replace}, 
one has that $\int_{0}^t\Omega(s)ds$ is monotonic. 
Then, it follows from Eqs.~\eqref{dot_V2}, \eqref{value_replace} that
\begin{align*}
\int_{0}^t\Omega(s)ds
\geq 
V(t)-\sum_{i\in{\cal V}}\int_{0}^t\bigg\{\frac{e_i(s)^2}{4}+\frac{\|\widetilde{d}_i(s)\|^2}{2}\bigg\}ds-V(0).
\end{align*}
Since the term 
$-\sum_{i\in{\cal V}}\int_{0}^t\big\{{e_i(s)^2}/{4}+{\| \widetilde{d}_i(s)\|^2}/{2}\big\}ds$
is lower bounded, and $V(0)$ and $V(t)$ are both bounded in Lemma~\ref{lemma_finiteescape}, one has 
$\int_{0}^t\Omega(s)ds$
is lower bounded as well, which implies that 
$\int_{0}^t\Omega(s)ds$ 
has a finite limit as $t\rightarrow\infty$.

Meanwhile, since  
$V(t)$ is bounded in Lemma~\ref{lemma_finiteescape}, it follows from Eq.~\eqref{V_1} that 
$\Phi_i, \widetilde{\omega}_i, \eta_i$ 
are all bounded. Combining with the boundedness of the first and second 
derivatives of 
$f_{i,j}(\omega_i), i\in{\cal V}, j\in\mathbb{Z}_1^n$
in condition C2, one has that 
$\dot{\Omega}$
is bounded as well, which implies that $\Omega$ in \eqref{value_replace} is uniformly
continuous in $t$. Then, it follows from Barbalat’s lemma \cite{khalil2002nonlinear} that
\begin{align}
\lim_{t\rightarrow\infty}\Omega(t)=0. 
\end{align}
Since
$a_i^2\geq0,\Phi_{i}\t K_iK_i\Phi_{i}\geq0, 
k_{i,j}>0, i\in\mathcal V, j\in\mathbb{Z}_1^n$ 
in Eqs.~\eqref{replace_val},~\eqref{value_replace}, 
one has
\begin{align}
\lim_{t\rightarrow\infty} a_i(t)= 0, 
\lim_{t\rightarrow\infty} \Phi_{i}(t)=\mathbf{0}_n,
\end{align}
which further implies $\lim_{t\rightarrow\infty} \phi_{i,j}(p_i(t)) =0, i\in\mathcal V, j\in\mathbb{Z}_1^n$, i.e., Claim 1) in Defintion~\ref{CPF_definition}.

Moreover, since $a_i$ in \eqref{replace_val} contains 
$\Phi_{i}(t), e_i(t)$ 
which both approach zeros when $t\rightarrow \infty$, one has that
\begin{align}
\label{condition}
\lim_{t\rightarrow\infty}c_i\widetilde{\omega}_i(t)-\eta_i(t) =0. 
\end{align}
It then follows from Eqs.~\eqref{dynamic_path_4} and \eqref{convergence_e} that
$\lim_{t\rightarrow\infty}\dot{\widetilde{\omega}}_i(t)=0$. 
From the fact $\dot{\omega}^{\ast}=(-1)^n$ and $\widetilde{\omega}_i=\omega_i-\omega^{\ast}$ in Eqs.~\eqref{target_velocity} and~\eqref{dynamic_path_4}, one has 
that $\lim_{t\rightarrow\infty} \dot{\omega}_{i}(t)=\lim_{t\rightarrow\infty} \dot{\omega}_{k}(t)\neq 0, \forall i\neq k \in{\cal V}$, i.e., the Claim 2) in Defintion~\ref{CPF_definition}. The proof is thus completed.
\end{proof}

\begin{remark}
From Lemmas~\ref{lemma_finiteescape} and \ref{lemma_pathconvergence}, the prevention of the finite-time-escape behavior and the global convergence of the robots to the common desired path can still be guaranteed in the presence of exponentially vanishing estimation errors and external disturbances in conditions C3 and C5 simultaneously. Moreover, the quantitative influence of the constant external disturbances on the {\it spontaneous-ordering} platoon is also demonstrated by numerical simulations in Section~\ref{simulation_23D} later.
\end{remark}

\begin{lemma}
\label{lemma_oderingflexible}
Under condition C4, a multi-robot system governed by \eqref{kinetic_F}, \eqref{desired_law} guarantees the {\it spontaneous-ordering} platoon, i.e., $r<\lim_{t\rightarrow\infty} |\omega_{s[k]}(t)-\omega_{s[k+1]}(t)|<R,  \forall k\in \mathbb{Z}_1^{N-1}$.
\end{lemma}

\begin{proof}
From the fact $\lim_{t\rightarrow\infty} \dot{\omega}_{i}(t)=\lim_{t\rightarrow\infty} \dot{\omega}_{k}(t)\neq 0, \forall i\neq k\in \mathcal{V}$
in Lemma~\ref{lemma_pathconvergence}, one has that the limiting 
relative value of 
$\omega_i, i\in{\cal V}$ 
against any 
$\omega_k, k\neq i$ 
is time-invariant with an arbitrary sequential ordering 
$\{s[1], s[2], \dots, s[N]\}$ 
in an ascending order, which satisfies $\omega_{s[1]}<\omega_{s[2]}<\dots<\omega_{s[N]}$.

Meanwhile, since
$|\omega_{i,k}(t)|>r, \forall i\neq k \in{\cal V}$ in Lemma~\ref{lemma_finiteescape}, 
one has that
$$
|\omega_{s[i],s[i+1]}|>r, i=1, \dots, n-1.
$$
Next, we will prove the condition of 
$|\omega_{s[i],s[i+1]}|<R, i=1, \dots, n-1$, 
by contradiction. With the loss of generality, we assume that there exists at least one pair of adjacent robots labeled $s[l],s[l+1]$ such that $|\omega_{s[l],s[l+1]}|\geq R$. 
Then, the contradiction is analyzed by the following three cases.

Case 1:  
$\omega_{s[l]}<\omega_{s[l+1]}\leq\omega^{\ast}$. 
As for the robot $s[l]$, one has that 
$-c_i(\omega_{s[l]}-\omega^{\ast})=-c_i\widetilde{\omega}_{s[l]}>0$.
Due to the assumption of
$|\omega_{s[l],s[l+1]}|\geq R$,
one has that 
$|\omega_{s[l],s[j]}|> R, j=l+1, \dots, n$,
which implies that robot $s[l]$ may only have neighbors satisfying $|\omega_{s[l],s[k]}|<R, k=l-1, \dots, 1$. 
It follows from the definition of 
$\eta_i$ in \eqref{de_eta} that $\eta_{s[l]}>0$,
which implies the limiting values $-c_i\widetilde{\omega}_{s[l]}+\eta_{s[l]}$ satisfy
\begin{align*}
-c_i\widetilde{\omega}_{s[l]}+\eta_{s[l]}
=&
c_i(\omega^{\ast}-{\omega}_{s[l]})+\eta_{s[l]}\nonumber\\
\geq& 
c_i(\omega_{s[l+1]}-{\omega}_{s[l]})+\eta_{s[l]}\nonumber\\
\geq&
c_iR>0.
\end{align*}
It contradicts Eq.~\eqref{condition}.

Case 2: 
$\omega^{\ast}\leq\omega_{s[l]}<\omega_{s[l+1]}$. 
As for robot $s[l+1]$, the contradiction is similar to robot 
$s[l]$ in case 1, one has that $-c_i\widetilde{\omega}_{s[l+1]}+\eta_{s[l+1]}\leq -c_iR<0$, which contradicts Eq.~\eqref{condition} as well.

Case 3: 
$\omega_{s[l]}<\omega^{\ast}<\omega_{s[l+1]}$.
As for robot $s[l]$, the contradiction is the same as case 1. 
As for robot $s[l+1]$, the contradiction is the same as case 2, 
of which are both omitted.

According to the contradiction of the cases $1,2,3$, one has that  $|\omega_{s[i],s[i+1]}|<R, i=1, \dots, n-1$.
Then, it is concluded that 
$r<|\omega_{s[i],s[i+1]}|<R, i=1, \dots, n-1$, i.e., Claim 3) in Definition~\ref{CPF_definition}. 
The proof is thus completed.
\end{proof}

\begin{remark}
The {\it spontaneous-ordering} property is achieved by the 
attraction of the target virtual coordinate $\omega^{\ast}$ and the repulsion among virtual coordinates $\omega_{i,k}$, of which both
finally reach a balance in one dimension (i.e., virtual coordinate $\omega$) and thus form the 
{\it spontaneous-ordering} platoon. Therein, the steady orderings of the platoon, however, are unknown in advance, which are distributively calculated during the multi-robot path-navigation process.
\end{remark}

\begin{theorem}
\label{theo_platoon}
A multi-robot system governed by \eqref{kinetic_F} and the DGVF algorithm \eqref{desired_law}
achieves the property P1, under the conditions C1, C2, C3, C4 and C5.
\end{theorem}
\begin{proof}
It follows from Lemmas~\ref{lemma_finiteescape}-\ref{lemma_oderingflexible} directly.
\end{proof}

\begin{remark}
Different from the string stability in previous platoon works \cite{ploeg2013controller,ploeg2013lp,besselink2017string,dunbar2011distributed,xu2022stochastic,feng2022robust,liu2020internal,qin2019experimental,monteil2019string,mokogwu2022energy,hu2020cooperative} which requires the robots to maneuver with fixed predecessor and follower neighbors (i.e., a platoon in terms of fixed-ordering string), the proposed DGVF~\eqref{desired_law} can handle time-varying neighbor relationships (i.e., the predecessor and follower of the robots cannot be uniquely determined and the string of the platoon is time-varying), which then enables the robots to form a {\it  spontaneous-ordering} platoon in the higher-dimensional Euclidean space ($n\geq 2$). So far, the string stability cannot be analyzed in the present {\it spontaneous-ordering} platoon with such time-varying predecessor and follower, which will be investigated in future work.
\end{remark}

\begin{remark}
The unwinding phenomenon commonly encountered in the rigid-robot attitude tracking problem, refers to the situation where a robot, whose attitude is represented by a quaternion, might perform an unnecessary large-angle maneuver, even if the initial attitude is close to the desired attitude \cite{dong2021anti1}. However, such an unwinding phenomenon is less relevant in this paper, because the proposed DGVF \eqref{desired_law} is designed and treated as the high-level desired guidance velocities (i.e., desired attitude) for simple single-integrator robots in Eq.~\eqref{kinetic_F}, rather than the low-level attitude tracking with rigid body dynamics. The “stretching” operation of the DGVF~\eqref{desired_law} in Fig.~\ref{illustration_HGVF} shows the unwinding effect in the end, and the robots may take a long way around the closed path to get into the platoon. We notice that some rigorous anti-unwinding techniques have been explored 
\begin{figure}[!htb]
\centering
\includegraphics[width=7cm]{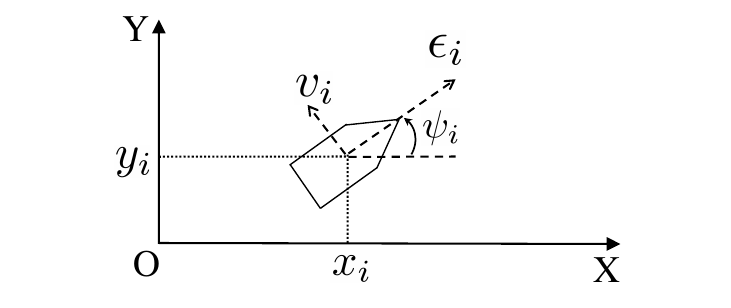}
\caption{ Illustration of the USV kinematics.}
\label{USV_dynamic}
\end{figure}
in the literature, such as the modified rodrigues parameters (MRPs) and sliding mode control (SMC)~\cite{dong2021anti1,dong2021anti2}, which can be seamlessly embedded in the low-level attitude tracking module with the desired attitude provided by the high-level DGVF \eqref{desired_law}.
\end{remark}

\section{2D Experimental Results and 3D Simulations}
In this section, we validate the effectiveness and robustness of the DGVF algorithm \eqref{desired_law} by 2D experiments using three HUSTER-0.3 USVs and 3D numerical simulations.

\subsection{Accommodating the DGVF to USV's Dynamics }
Since the DGVF~algorithm~\eqref{desired_law} provides high-level reference tracking velocities rather than low-level control signals when encountering robots with high-order dynamics, it applies to any robots whose guidance velocities can be exponentially tracked with well-designed low-level motor control signals. 
In what follows, we will first introduce the accommodating of the DGVF algorithm \eqref{desired_law} to the USVs.
The kinematics of USV $i$ in the Cartesian coordinates~\cite{hu2021bearing} are, 
\begin{align}
\label{kinetic_USV}
 \dot{x}_i &=\epsilon_i\cos\psi_i-v_i\sin\psi_i, \nonumber\\
\dot{y}_i &=\epsilon_i\sin\psi_i+v_i\cos\psi_i, \nonumber\\
\dot{\psi}_i &=r_i
\end{align}
with the positions $q_i(t)=[x_i(t), y_i(t)]\t\in \mathbb{R}^2$, the yaw angle  $\psi_i(t) \in \mathbb{R}$ in the Cartesian coordinate, and $\epsilon_i(t), v_i(t), r_i(t) \in \mathbb{R}$ the surge, the sway and the yaw velocities of USV $i$ in the USV coordinate, respectively, as shown in Fig.~\ref{USV_dynamic}.

The dynamics of USV $i$ are described by a practical model (see e.g., \cite{Liubin2018surounding}) 
\begin{align}
\label{kinematic_F}
\dot{\epsilon}_i &=l_1\epsilon_i+l_2v_i r_i+l_3\tau_{i,1}, \nonumber\\
\dot{r}_i &=l_4r_i+l_5\tau_{i,2}, \nonumber\\
\dot{v}_i &=l_6v_i+l_7\epsilon_i r_i,
\end{align}
where $l_1, l_2, l_3, l_4, l_5, l_6, l_7\in \mathbb{R}$ are the identified parameters, and $\tau_{i,1}, \tau_{i,2}\in \mathbb{R}$ the actuator inputs of USV~$i$.
It follows from Eq.~\eqref{kinetic_USV} that $\dot{x}_i, \dot{y}_i$ can be rewritten in a compact form,
\begin{align}
\label{part_kinetic_F}
\begin{bmatrix}
\dot{x}_i\\
\dot{y}_i
\end{bmatrix}
=
\begin{bmatrix}
\cos\psi_i & -\sin\psi_i\\
\sin\psi_i & \cos\psi_i
\end{bmatrix}
\begin{bmatrix}
\epsilon_i\\
v_i
\end{bmatrix}.
\end{align}

Analogously, substituting Eq.~\eqref{part_kinetic_F} into~the closed-loop system~\eqref{dynamic_path} yields
\begin{align}
\label{dynamic_path_1}
\begin{bmatrix}
   \dot{\phi}_{i,1}\\
   \dot{\phi}_{i,2}\\
   \dot{\omega}_i\\
\end{bmatrix} =& 
		      \begin{bmatrix}                      
                       1 & 0 & -\partial f_{i,1}(\omega_i)\\
                       0 & 1 & -\partial f_{i,2}(\omega_i)\\
                       0 & 0 &1\\
                       \end{bmatrix}
                       \begin{bmatrix}                      
                       \cos\psi_i & -\sin\psi_i & 0\\
                       \sin\psi_i & \cos\psi_i & 0\\
                       0 & 0 &1\\
                       \end{bmatrix}
                       \begin{bmatrix}   
                       \epsilon_i\\
                       v_i\\
                       \dot{\omega}_i\\
                       \end{bmatrix}.   
\end{align}
Let $\epsilon_i^r, v_i^r$ be the high-level guidance velocities for $\epsilon_i, v_i$, respectively. Defining the signal errors as 
\begin{figure}[!htb]
\centering
\includegraphics[width=\hsize]{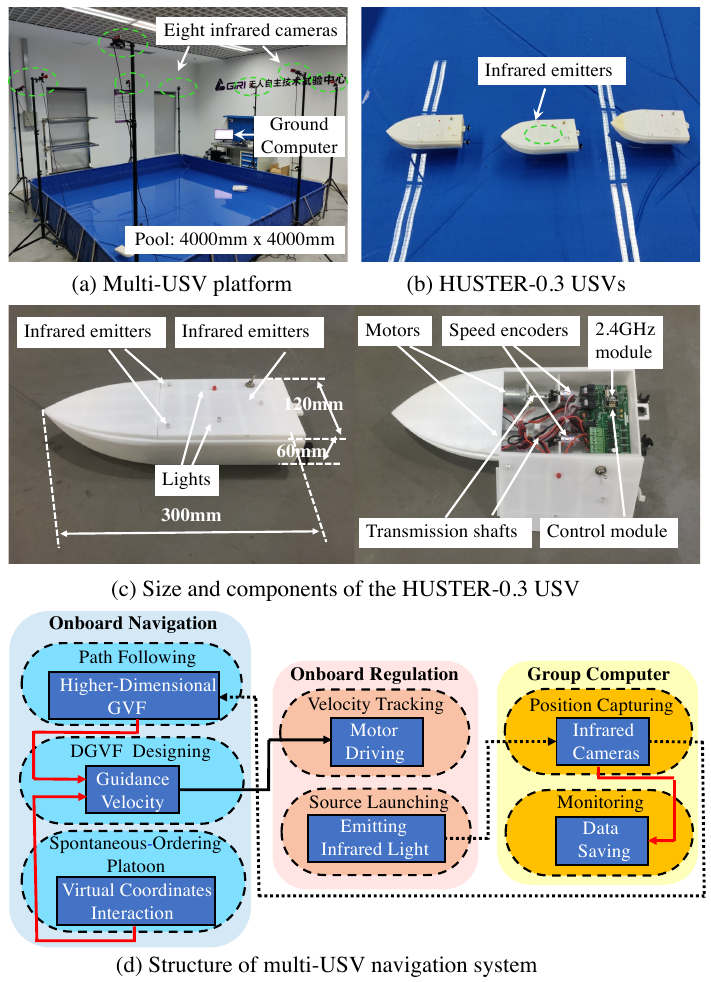}
\caption{ (a) The multi-USV platform consists of eight infrared cameras, a computer, and a $4000$ mm $\times$ $4000$ mm pool. (b) Three HUSTER-0.3 USVs with infrared emitters as the identifiers on their tops which can be identified by eight infrared cameras during the experiments. (c) Size: $300$mm (length) $\times$ $120$mm (width) $\times$ $60$mm (height) and detailed components of the HUSTER-0.3 USV. (d) Structure of the multi-USV navigation system, where the path following, {\it spontaneous-ordering} platoon, DGVF designing, velocity tracking, and source launching are running onboard, and the position capturing and monitoring are conducted on the ground computer. (The solid lines denote physical connections whereas the dotted lines virtual connections.)  
  }
\label{MUSV_platform}
\end{figure}
\begin{align}
\label{velocity_error}
\widetilde{\epsilon}_i:=\epsilon_i-\epsilon_i^r, \widetilde{v}_i:=v_i-v_i^r,
\end{align}  
it follows from Eqs.~\eqref{dynamic_path_1} and \eqref{velocity_error} that
\begin{align}
\label{dynamic_path_2}
\begin{bmatrix}
   \dot{\phi}_{i,1}\\
   \dot{\phi}_{i,2}\\
   \dot{\omega}_i\\
\end{bmatrix} =& \begin{bmatrix}                      
                       \cos\psi_i & -\sin\psi_i & -\partial f_{i,1}(\omega_i)\\
                       \sin\psi_i & \cos\psi_i & -\partial f_{i,2}(\omega_i)\\
                       0 & 0 &1\\
                       \end{bmatrix}
                       \begin{bmatrix}   
                       \epsilon_i^r\\
                       v_i^r\\
                       u_i^{\omega}\\
                       \end{bmatrix} 
                       +
                        \begin{bmatrix}   
                         e_{\epsilon,i}\\
                         e_{v,i}\\
                       0\\
                       \end{bmatrix}     
\end{align}
with
\begin{align}
\label{disturbance_e}
e_{\epsilon,i}:=&\widetilde{\epsilon}_i\cos\psi_i-\widetilde{v}_i\sin\psi_i,\nonumber\\
e_{v,i}:=&\widetilde{\epsilon}_i\sin\psi_i+\widetilde{v}_i\cos\psi_i.
\end{align}

Starting from DGVF~\eqref{desired_law} for the single-integrator robots in~\eqref{kinetic_F}, a modified DGVF algorithm for USV $i$ is naturally proposed as follows,
\begin{align}
\label{desired_USV_law}
\epsilon_i^r=&(\partial f_{i,1}-k_{i,1}\phi_{i,1})\cos\psi_i+(\partial f_{i,2}-k_{i,2}\phi_{i,2})\sin\psi_i,\nonumber\\
v_i^r=&-(\partial f_{i,1}-k_{i,1}\phi_{i,1})\sin\psi_i+(\partial f_{i,2}-k_{i,2}\phi_{i,2})\cos\psi_i,\nonumber\\
u_i^{\omega}=&1+\sum\limits_{j=1}^2 k_{i,1}\phi_{i,1}\partial f_{i,1}-c_i(\omega_i-\widehat{\omega}_i)+\eta_i.
\end{align}
Note that the low-level velocity tracking problem in \eqref{kinematic_F} and~\eqref{velocity_error}, i.e., $\lim_{t\rightarrow\infty}\widetilde{\epsilon}_i(t)=0, \lim_{t\rightarrow\infty}\widetilde{v}_i(t)=0$ exponentially has been well addressed in \cite{hu2021bearing}, which then follows from~\eqref{disturbance_e} and the bounded trigonometric function $\cos\psi_i, \sin\psi_i$ that $\lim_{t\rightarrow\infty}e_{\epsilon,i}(t)=0, \lim_{t\rightarrow\infty}e_{v,i}(t)=0$ exponentially.

\begin{proposition}
Under the conditions C1-C4, a multi-USV system composed of \eqref{kinetic_USV},~\eqref{kinematic_F} and the modified DGVF algorithm~\eqref{desired_USV_law}
achieves the property P1 subject to $\lim_{t\rightarrow\infty}\widetilde{\epsilon}_i(t)=0, 
\lim_{t\rightarrow\infty}\widetilde{v}_i(t)=0$, exponentially. 
\end{proposition}
\begin{proof}
The proof is similar to Theorem~\ref{theo_platoon}, which is thus omitted.
\end{proof}

\subsection{2D Experiments with USVs}
For the experiments, we firstly establish an indoor multi-USV navigation platform and thereby conduct the {\it spontaneous-ordering} platoon experiments. As shown in Fig.~\ref{MUSV_platform}~(a), the multi-USV navigation platform is composed of a $4000$mm $\times$ $4000$mm pool, a motion-capture system (eight Flex 3 infrared cameras) to identify the positions of the USVs, and a ground computer (Intel core i7-960) to transmit, analyze and store the detection data to the three HUSTER-0.3 USVs. As demonstrated in Fig.~\ref{MUSV_platform}~(b), three HUSTER-0.3 USVs are all equipped with infrared emitters, which are utilized for identification by infrared cameras. 
Moreover, it is observed in Fig.~\ref{MUSV_platform} (c) that each HUSTER-0.3 USV is $300$mm in length, $120$mm in width, and $60$mm in height, which is equipped with two DC motors ($5$V), two speed encoders (Mini-256), two transmission shafts ($150$mm $\times$ $6$mm), a control module (STM32F1) and a $2.4$GHz wireless module (NRF24L01). Please refer to our previous work \cite{hu2021bearing} for more details. Fig.~\ref{MUSV_platform}~(d) exhibits the structure of the multi-USV navigation system, which is divided into three parts: the onboard navigation to produce desired guidance velocity based on DGVF, the onboard regulation to track velocity and launch infrared lights, and the ground computer to capture positions and save data.  
During the navigation experiments, our DGVF algorithm is running with a fixed $10$Hz frequency and all the data are transmitted to and saved on the ground computer.

In what follows, we consider the 2D circular and {\it self-intersecting} Lissajous waterway  (i.e., desired paths) to conduct {\it spontaneous-ordering} platoon experiments using the modified DGVF~\eqref{desired_USV_law}. 
First, we choose the sensing 
and safe radius $R=1.0, r=0.7$, where the potential function $\alpha(s)$
can be designed based on Eq.~\eqref{potential_alpha}.
The target virtual coordinate $\omega^{\ast}$ satisfies
$\dot{\omega}^{\ast}=1$ in Eq.~\eqref{target_velocity}, where the initial value $\omega^{\ast}(0)$ is set to be $\omega^{\ast}(0)=0$. By 
Remark~\ref{remark_c3}, we pick the estimator gains $\gamma_1=20, \gamma_2=4$ 
to satisfy condition C3 with a connected communication topology.

\begin{figure}[!htb]
\centering
\includegraphics[width=\hsize]{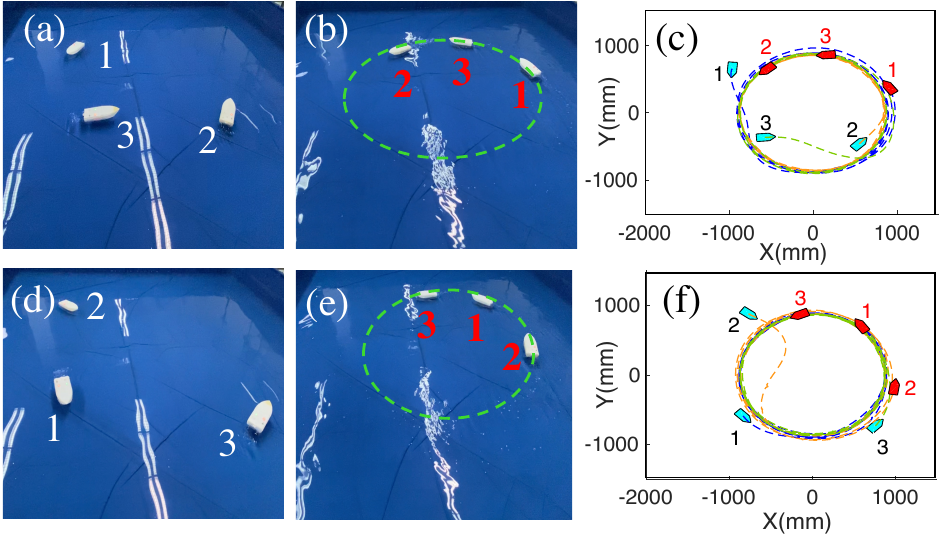}
\caption{ Two experimental cases of the {\it spontaneous-ordering} platoon moving along the desired 2D circular waterway  using the modified DGVF~\eqref{desired_USV_law}. Subfigures (a), (d): Initial positions of the USVs. Subfigures (b), (e): Final platoons move along the circular waterway  after 40 seconds. Subfigures (c), (f): Trajectories of the three USVs from the initial positions to the final platoon with distinct ordering sequences (Here, the blue vessels represent the initial positions, and the red ones the final platoon).}
\label{experi_snap_circle}
\end{figure}
\begin{figure}[!htb]
\centering
\includegraphics[width=8.25cm]{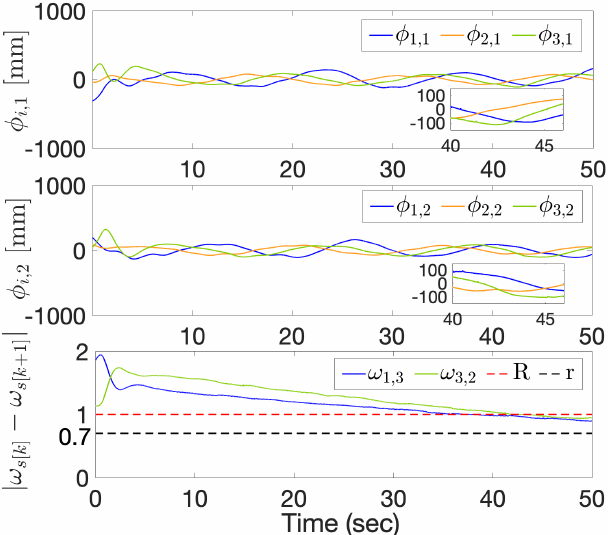}
\caption{ Temporal evolution of the position errors 
$\phi_{i,1}, \phi_{i,2}, i=1,2,3,$ and the final relative value of virtual coordinates between each pair of adjacent robots $|\omega_{s[k]}(t)-\omega_{s[k+1]}(t)|, k=1, 2,$
in Fig.~\ref{experi_snap_circle} (c).}
\label{experi_circle_performance}
\end{figure}

\begin{figure}[!htb]
\centering
\includegraphics[width=\hsize]{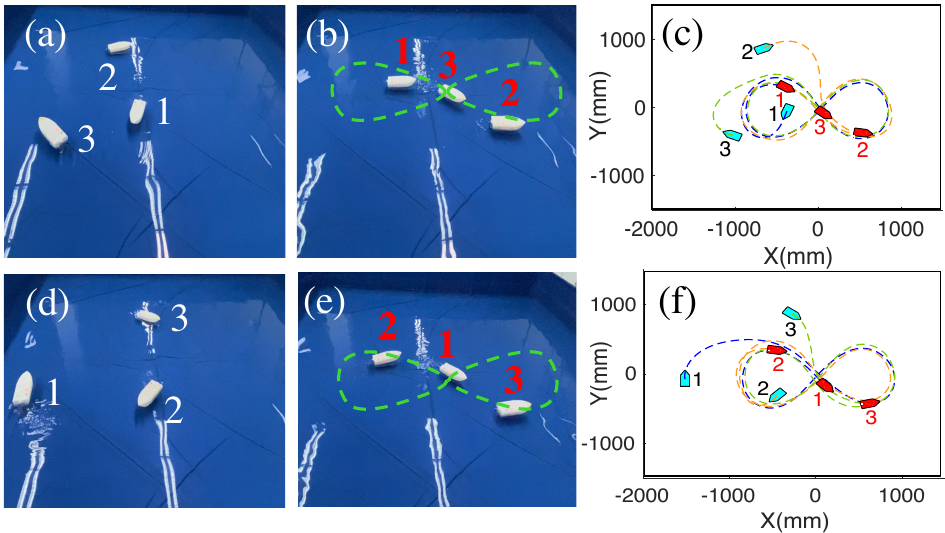}
\caption{ Two experimental cases of the {\it spontaneous-ordering} platoon moving along the desired  2D {\it self-intersecting} waterway  using the proposed modified DGVF algorithm~\eqref{desired_USV_law}. Subfigures (a), (d): Initial positions of the USVs. Subfigures (b), (e): Final platoon moves along the {\it self-intersecting} waterway after 16s. Subfigures (c), (f): Trajectories of the three USVs from the initial positions to the final platoon with distinct ordering sequences (Here, the blue vessels represent the initial positions, and the red ones the final platoon).}
\label{experi_snap_eight}
\end{figure}
\begin{figure}[!htb]
\centering
\includegraphics[width=8.4cm]{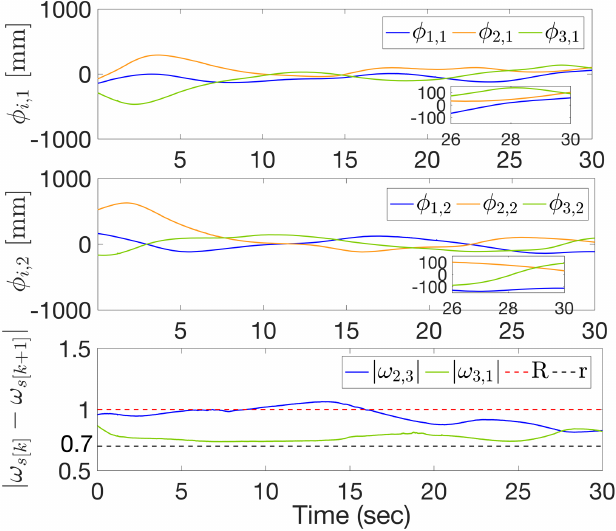}
\caption{ Temporal evolution of the position errors 
$\phi_{i,1}, \phi_{i,2}, i=1,2,3,$ and the final relative value of virtual coordinates between each pair of adjacent robots $|\omega_{s[k]}(t)-\omega_{s[k+1]}(t)|, k=1, 2,$
in Fig.~\ref{experi_snap_eight} (c).}
\label{experi_eight_performance}
\end{figure}

For the desired 2D circular paths $\mathcal P_i^{phy}, i\in{\cal V}$, the parametrization is 
$$x_{i,1}=800\cos\omega_i\mbox{mm},~x_{i,2}=800\sin\omega_i \mbox{mm},$$
which fulfills conditions C2 and C4. We choose the gains $k_{i,1}=3.5, k_{i,2}=3.5, c_i=2$  in \eqref{desired_USV_law}. Fig.~\ref{experi_snap_circle} illustrates two experimental cases of the {\it spontaneous-ordering} platoon moving along the common desired 2D circular waterway. As shown in Figs.~\ref{experi_snap_circle} (c) and (f), three USVs from different initial positions (blue vessels) achieve platoons (red vessels) with distinct ordering sequences (Fig.~\ref{experi_snap_circle} (c): \{2,3,1\} and Fig.~\ref{experi_snap_circle} (f): \{3,1,2\}), where the corresponding experimental snapshots of the initial positions and the final platoons are given in Figs.~\ref{experi_snap_circle} (a), (b), (d), (e), respectively. It thus verifies that the ordering of the platoon is spontaneous. Additionally, we take Fig.~\ref{experi_snap_circle} (c) as an example to analyze the state evolution in the circular-path experiments. 
 It is observed in the zoomed-in panels $[40\mbox{s},45\mbox{s}]\times [-100\mbox{mm}, 100\mbox{mm}]$ of Fig.~\ref{experi_circle_performance}  that $\phi_{i,1}, \phi_{i,2}, i=1,2,3,$ approach and stay in the range of $[-100\mbox{mm}, 100\mbox{mm}]$ after $40$ seconds, which is acceptable compared with the length of the desired circular path and the size of the USV in the trajectory of Fig.~\ref{experi_snap_circle} (c). In this way, the effectiveness of tracking the desired circular waterway is verified.
Moreover, the relative value of adjacent virtual coordinates $|\omega_{1,3}|, |\omega_{3,2}|$ satisfy $|\omega_{1,3}|\in (0.7, 1.0), |\omega_{3,2}|\in (0.7, 1.0)$ after $40$ seconds in 
Fig.~\ref{experi_circle_performance}, which verifies that the platoon in terms of relative parametric displacement is achieved. The feasibility of the proposed algorithm~\eqref{desired_USV_law} for closed waterways  is thus demonstrated.

\begin{figure}[!htb]
\centering
\includegraphics[width=\hsize]{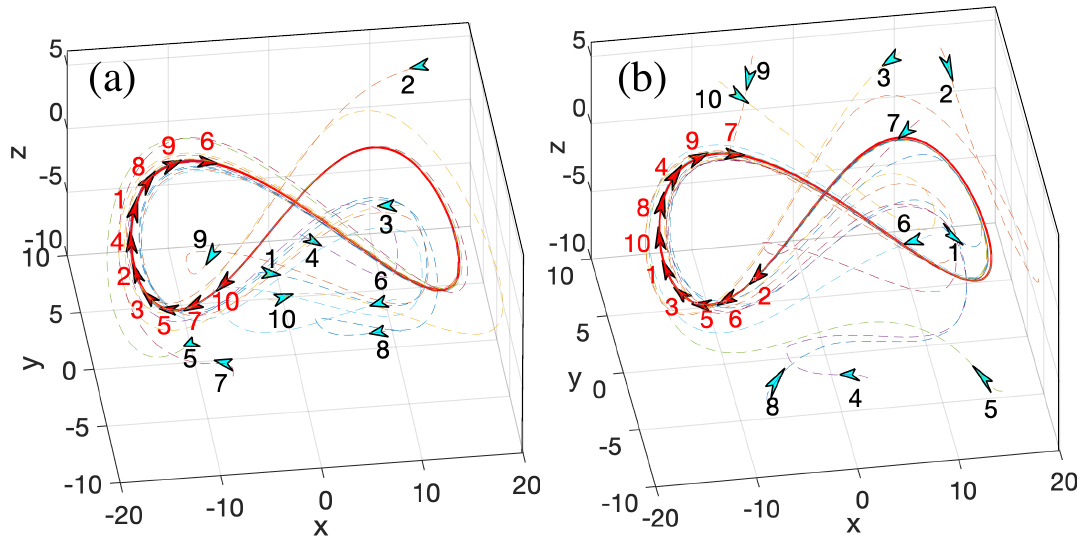}
\caption{
Cases (a)-(b): trajectories of ten robots from different initial positions to {\it spontaneous-ordering} platoons moving along a desired Lissajous path in the 3D Euclidean space with the proposed DGVF~\eqref{desired_law}. (Here, the blue and red arrows represent the initial and final positions of the robots, respectively. The red line denotes the desired Lissajous path).}
\label{Lissajous_12}
\end{figure}
\begin{figure}[!htb]
\centering
\includegraphics[width=8cm]{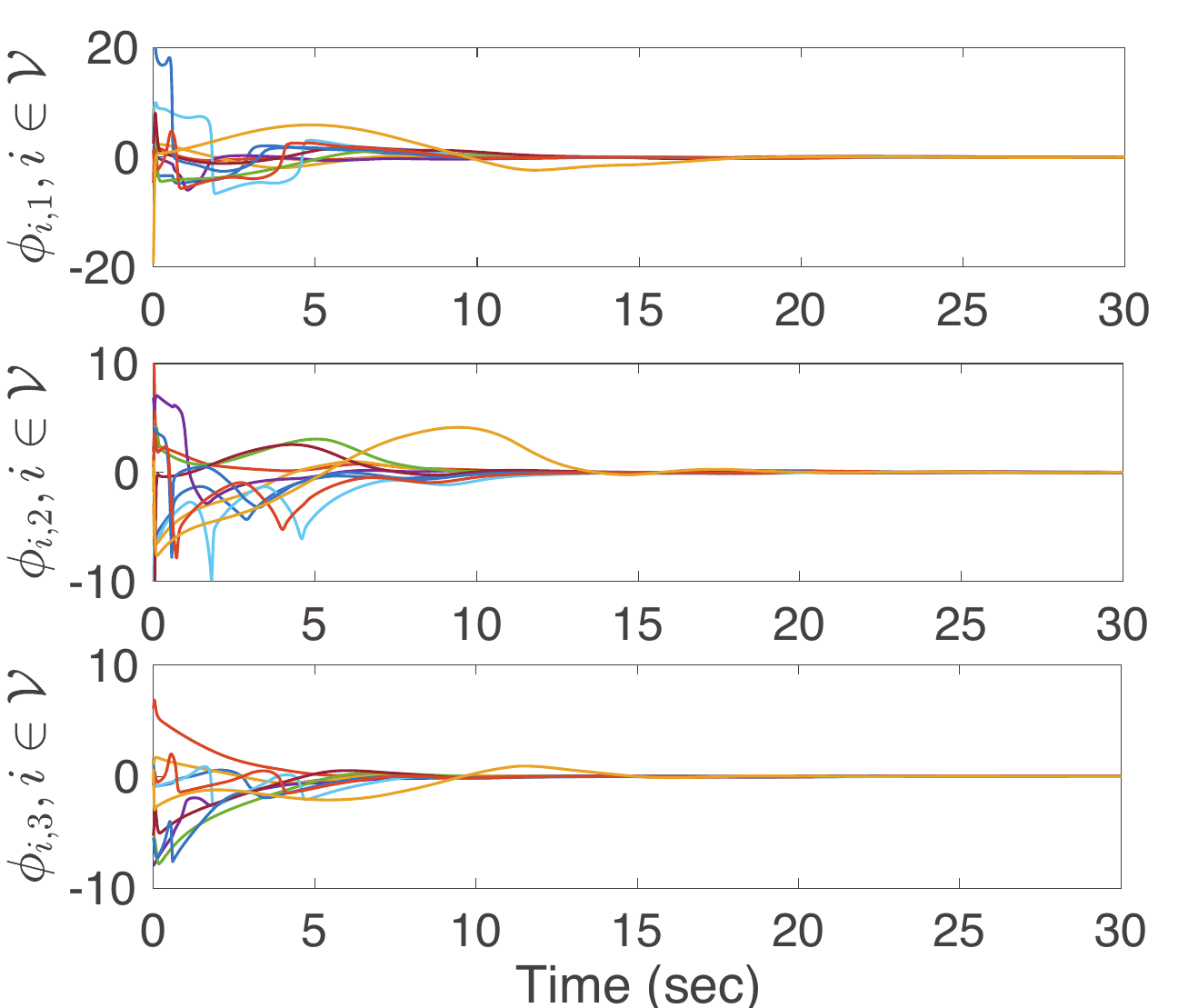}
\caption{Temporal evolution of the position errors 
$\phi_{i,1}, \phi_{i,2}, \phi_{i,3}, \forall i\in\mathbb{Z}_1^{10},$
in Fig.~\ref{Lissajous_12} (a) for example.}
\label{self_intersected_performance1}
\end{figure}
\begin{figure}[!htb]
\centering
\includegraphics[width=8cm]{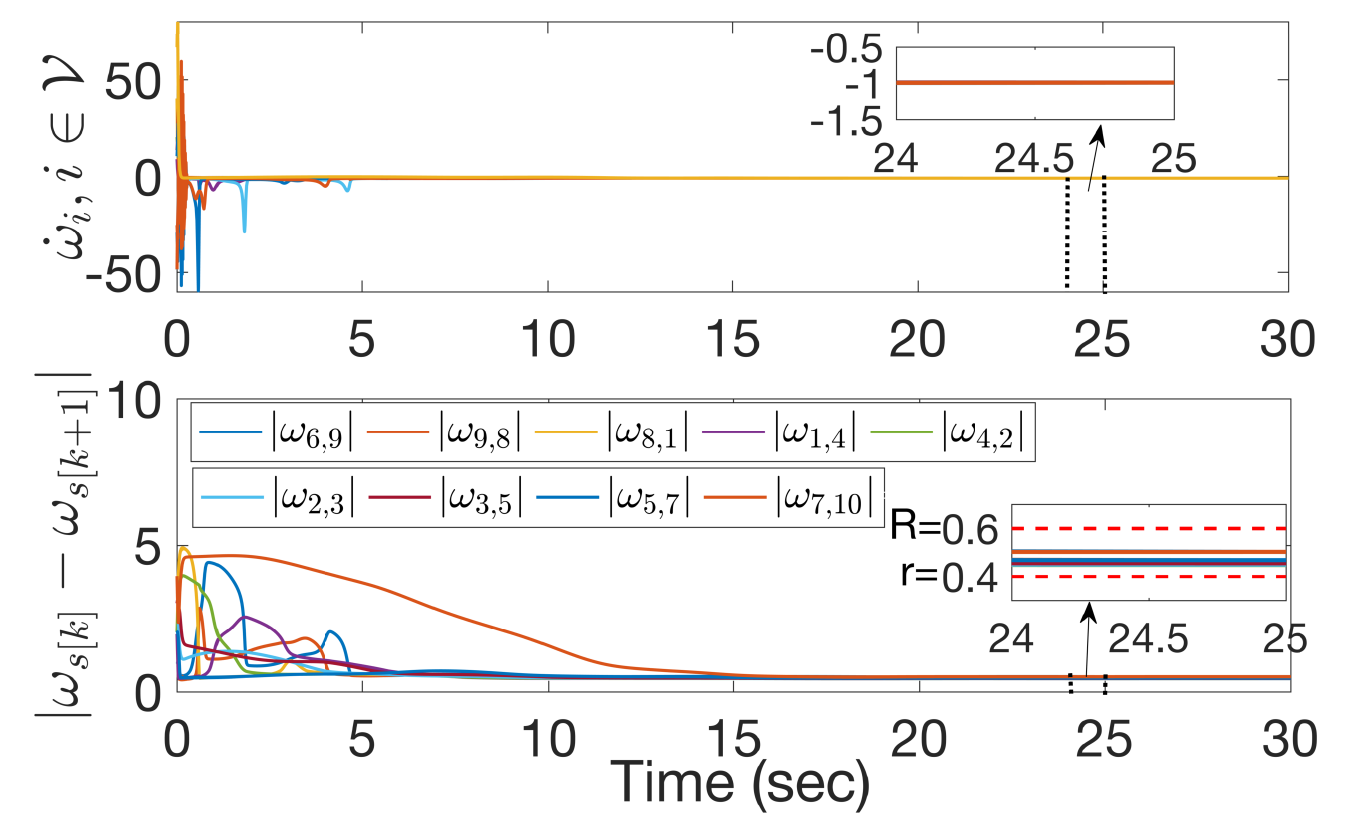}
\caption{ Temporal evolution of the derivative of the virtual coordinate $\dot{\omega}_i, \forall i\in\mathbb{Z}_1^{10}$ 
and the final relative value of virtual coordinates between adjacent robots $|\omega_{s[k]}(t)-\omega_{s[k+1]}(t)|, \forall k\in\mathbb{Z}_1^{9},$
in Fig.~\ref{Lissajous_12} (a) for example.}
\label{self_intersected_path2}
\end{figure}

\begin{figure}[!htb]
\centering
\includegraphics[width=\hsize]{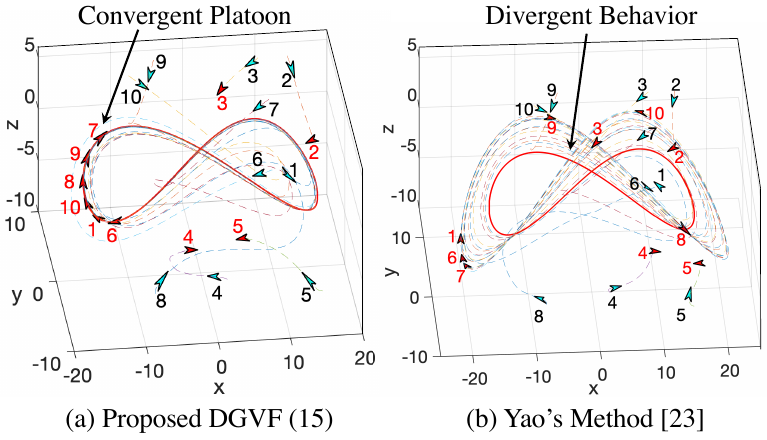}
\caption{ A special situation of four robots $i=2,3,4,5,$ suddenly breaking down at $t = 2$s and the rest of six robots stoping interacting with the broken four robots when $t>2$s in the ten-robot path navigation mission. Trajectory and platoon comparison of the rest of the six robots $i=1,6,7,8,9,10,$
between the proposed DGVF \eqref{desired_law} (see the successful platoon in subfigure (a)) and Yao's fixed-ordering method~\cite{yao2021distributed} (see the failure of a platoon in subfigure (b)). (Here, the blue and red arrows represent the initial and final positions of the robots, respectively. The red line denotes the desired Lissajous path).
}
\label{broken_Lissajous_12}
\end{figure}

For the desired 2D Lissajous path $\mathcal P_i^{phy}, i\in{\cal V}$ containing {\it self-intersecting} points, it follows from the conditions C2 and C4 that the parametrization is
\begin{align*}
x_{i,1}=\frac{800\cos\omega_i}{1+0.3(\sin\omega_i)^2} \mbox{mm}, x_{i,2}=\frac{800\sin\omega_i\cos\omega_i}{1+0.3(\sin\omega_i)^2}\mbox{mm}
\end{align*}
with the virtual coordinates $\omega_i$. The gains in \eqref{desired_USV_law} are set to be $k_{i,1}=2, k_{i,2}=2, c_i=2$. Analogously, Fig.~\ref{experi_snap_eight} illustrates two experimental cases of the {\it spontaneous-ordering} platoon whereas moving along the common desired 2D {\it self-intersecting} waterway. It is observed in Figs.~\ref{experi_snap_eight} (c) and (f) that three USVs from different initial positions (blue vessels) also achieve platoons (red vessels) with distinct ordering sequences (Fig.~\ref{experi_snap_eight} (c): \{2,3,1\} and Fig.~\ref{experi_snap_eight} (f): \{3,1,2\}), where the corresponding experimental snapshots of initial positions and final platoons are given in Figs.~\ref{experi_snap_eight} (a), (b), (d), (e), respectively. We take Fig.~\ref{experi_snap_eight} (c) as an example to analyze the state evolution in the {\it self-intersecting}-waterway experiments.

As shown in the zoomed-in panels $[26\mbox{s},30\mbox{s}]\times [-100\mbox{mm}, 100\mbox{mm}]$ of Fig.~\ref{experi_eight_performance}, $\phi_{i,1}, \phi_{i,2}, i=1,2,3,$ approach and stay in the range of $[-100\mbox{mm}, 100\mbox{mm}]$ after $26$ seconds, which is acceptable as well compared with the length of the desired Lissajous path and the size of the USV in the trajectory of Fig.~\ref{experi_snap_eight} (c). It thus verifies the effectiveness of tracking the desired Lissajous waterway. Moreover, the relative value of adjacent virtual coordinates $|\omega_{1,3}|, |\omega_{3,2}|$ also satisfy $|\omega_{1,3}|\in (0.7, 1.0), |\omega_{3,2}|\in (0.7, 1.0)$ after $16$ seconds in Fig.~\ref{experi_eight_performance}, which verifies that the platoon in terms of relative parametric displacement is also achieved. The feasibility of the proposed DGVF algorithm~\eqref{desired_USV_law} for {\it self-intersecting} waterway  is thus substantiated. More experimental details can be viewed in the attached video. \footnote{Online. Available: \href{https://www.youtube.com/watch?v=QCkpw6Pwpoo}{https://www.youtube.com/watch?v=QCkpw6Pwpoo} }

\begin{figure}[!htb]
\centering
\includegraphics[width=\hsize]{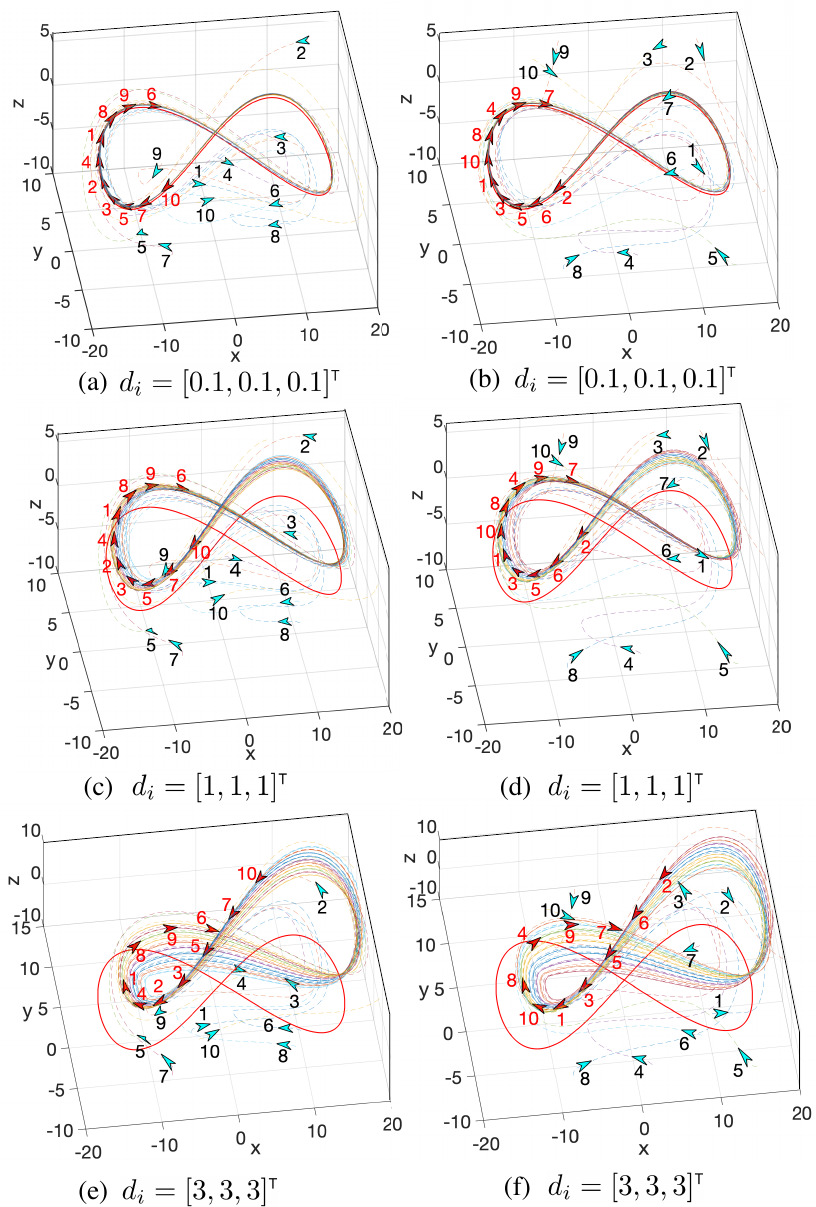}
\caption{Trajectories and platoon performance of ten robots governed by the proposed DGVF~\eqref{desired_law} under external constant disturbances with
increasing intensities. Subfigures (a) and (b): $d_i=[0.1,0.1,0.1]\t, i\in\mathbb{Z}_1^{10}$, subfigures (c) and (d): $d_i=[1, 1, 1]\t, i\in\mathbb{Z}_1^{10}$, and subfigures (e) and (f): $d_i=[3, 3, 3]\t, i\in\mathbb{Z}_1^{10}$.
(Here, the blue and red arrows represent the initial and final positions of the robots, respectively. The red line denotes the desired Lissajous path).
}
\label{distur_1_Lissajous_12}
\end{figure}

\subsection{3D Numerical Simulations}
\label{simulation_23D}
In this part, 3D numerical simulations are conducted to validate the feasibility of Theorem~\ref{theo_platoon} in the higher-dimensional Euclidean space.
We consider $n=10$ robots governed by~\eqref{kinetic_F} and \eqref{desired_law}, where the sensing 
and safe radius are given by $R=0.6, r=0.4$, respectively. The potential function $\alpha(s)$ is designed according to Eq.~\eqref{potential_alpha}. 
Moreover, the parametric setting for the target virtual coordinate $\omega^{\ast}$ and the estimator gains $\gamma_1, \gamma_2$ in Remark~\ref{remark_c3} are the same as those in the experimental subsection.

\begin{figure}[!htb]
\centering
\includegraphics[width=8cm]{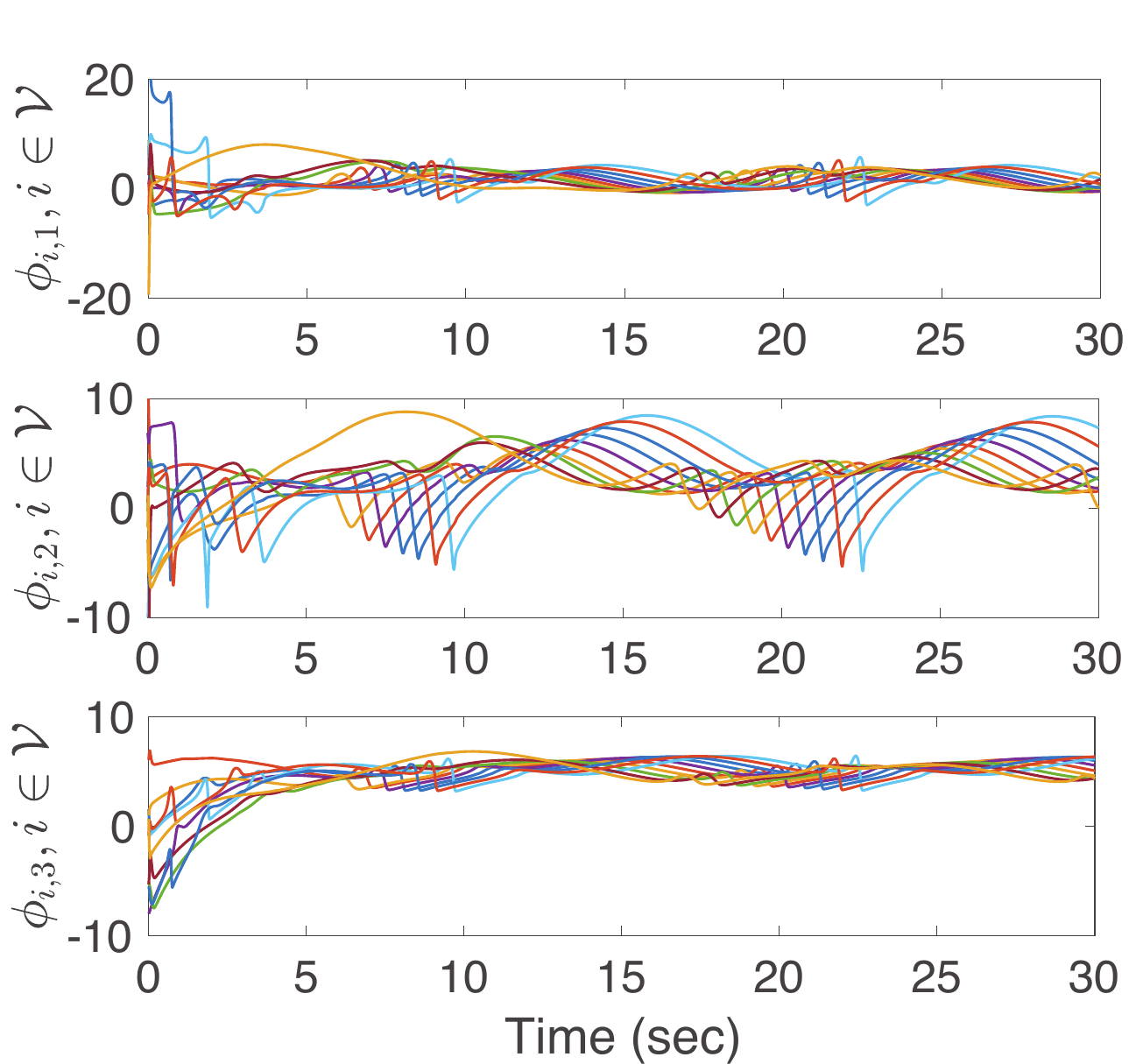}
\caption{Temporal evolution of the position errors 
$\phi_{i,1}, \phi_{i,2}, \phi_{i,3}, \forall i\in\mathbb{Z}_1^{10},$ under the constant disturbances $d_i=[3,3,3]\t, i\in\mathcal V,$
in Fig.~\ref{distur_1_Lissajous_12} (e) for example.}
\label{Dis_Lissajous_performance1}
\end{figure}

\begin{figure}[!htb]
\centering
\includegraphics[width=8cm]{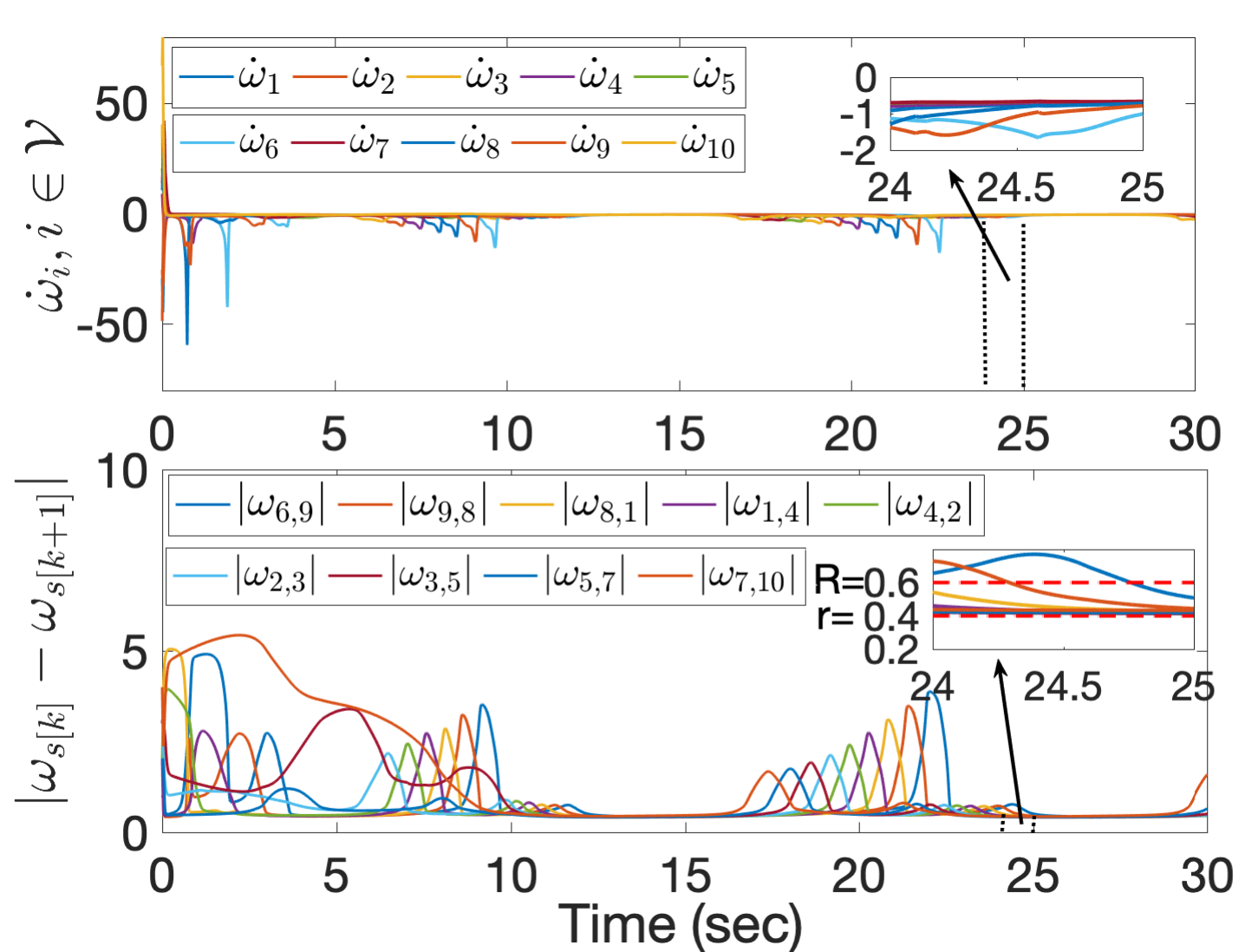}
\caption{ Temporal evolution of the derivative of the virtual coordinate $\dot{\omega}_i, \forall i\in\mathbb{Z}_1^{10}$ 
and the final relative value of virtual coordinates between adjacent robots $|\omega_{s[k]}(t)-\omega_{s[k+1]}(t)|, \forall k\in\mathbb{Z}_1^{9},$ under the constant disturbances $d_i=[3,3,3]\t, i\in\mathcal V,$
in Fig.~\ref{distur_1_Lissajous_12} (e) for example.}
\label{Dis_Lissajous_performance2}
\end{figure}

In what follows, a desired 3D Lissajous path 
$\mathcal P_i^{phy}, i\in{\cal V}$ containing {\it self-intersecting} points is considered, 
of which the parametrization is 
$x_{i,1}=16\cos(0.5\omega_i), x_{i,2}=6\cos(\omega_i+\frac{\pi}{2}), x_{i,3}=2\cos\omega_i, i\in\mathcal V$
fulfilling conditions C2 and C4.
The parameters in \eqref{desired_law} are set to be 
$k_{i,1}=0.6, k_{i,2}=0.6, k_{i,3}=0.6, c_i=3, i\in{\cal V}$.
Figs.~\ref{Lissajous_12} (a)-(b) describes the trajectories of 
ten robots from different initial positions (blue arrows) fulfilling the condition C1
to the {\it spontaneous-ordering} platoon maneuvering along the common 3D Lissajous path (red arrows) in the 3D Euclidean space.

During the process, the multi-robot platoon   
with different initial positions is achieved with distinct ordering sequences (Fig.~\ref{Lissajous_12} (a): \{6,9,8,1,4,2,3,5,7,10\} and Fig.~\ref{Lissajous_12} (b): \{7,9,4,8,10,1,3,5,6,2\}), which demonstrates the property of spontaneous orderings as well. 
Additionally, we take Fig.~\ref{Lissajous_12} (b) as an illustrative example to analyze the states' evolution of {\it spontaneous-ordering} platoon in the {\it self-intersecting-path} navigation task. As shown in Fig.~\ref{self_intersected_performance1}, the path-following errors $\phi_{i,1},\phi_{i,2}, \phi_{i,3}, i\in {\cal V}, $ converge to zeros after 17 seconds, i.e., $\lim_{t\rightarrow\infty}\phi_{i,1}(t)=0, \lim_{t\rightarrow\infty}\phi_{i,2}(t)=0, \lim_{t\rightarrow\infty}\phi_{i,3}(t)=0, i\in{\cal V}$, which verifies Claim 1) in Definition~\ref{CPF_definition}. 
Essentially, it is observed in Fig.~\ref{self_intersected_path2} that $\lim_{t\rightarrow\infty}\dot{\omega}_i(t)=1, i\in{\cal V}$, which verifies $\lim_{t\rightarrow\infty} \dot{\omega}_{i}(t)=\dot{\omega}_{k}(t)\neq 0, \forall i\neq k \in{\cal V}$ in Claim 2) of Definition~\ref{CPF_definition} explicitly. 
The final relative values of adjacent virtual coordinates satisfy $0.4<\lim_{t\rightarrow\infty} |\omega_{s[k]}(t)-\omega_{s[k+1]}(t)|<0.6, \forall k\in\mathbb{Z}_1^{10}$ in the zoomed-in panels $[24\mbox{s}, 25\mbox{s}]\times [0.3, 0.7]$ of Fig.~\ref{self_intersected_path2}, then it fulfills Claims 3) and 4) of Definition 2. It thus has the property~P1.

To show the robustness of the proposed DGVF algorithm, we compare our algorithm with Yao's fixed-ordering method \cite{yao2021distributed} when some robots break down during the 
multi-robot path navigation process. According to the fixed-ordering method in \cite{yao2021distributed}, we set the desired values between adjacent virtual coordinates to be $|\omega_{k, k+1}^*|=2\pi/15, k\in\mathbb{Z}_1^9$ with a connected and fixed communication topology in advance, which can form a platoon with a fixed ordering sequence: $\{1,2,3,4,5,6,7,8,9,10\}$.
Then, we consider a special situation when four robots $i=2,3,4,5,$ suddenly break down at $t=2$s and conduct the numerical simulations to compare the robustness between the proposed DGVF \eqref{desired_law} and the fixed-ordering algorithm  \cite{yao2021distributed}. As shown in Fig.~\ref{broken_Lissajous_12} (a), the rest of six robots $i=1,6,7,8,9,10$ governed by DGVF \eqref{desired_law} stop communicating with the broken robots $i=2,3,4,5$ and only interacting with the rest ones, which still forms a satisfactory six-robot platoon fulfilling the property P1. However, it is observed in Fig.~\ref{broken_Lissajous_12} (b) that the rest of six robots governed by the method in \cite{yao2021distributed} fail to form the platoon anymore. The robustness of the DGVF design \eqref{desired_law} is
thus verified when experiencing robots breakdown.

Moreover, to analyze the quantitative influence of external disturbances~$d_i$ in \eqref{kinetic_F} on the DGVF algorithm \eqref{desired_law}, we consider the external constant disturbances with increasing intensities in the {\it spontaneous-ordering} platoon task. 
For $d_i=[0.1, 0.1, 0.1]\t, i\in\mathbb{Z}_1^{10}$, Fig.~\ref{distur_1_Lissajous_12} (a)-(b) illustrates that ten robots from different initial positions (blue arrows) can still form the {\it spontaneous-ordering} platoon maneuvering along the common desired 3D Lissajous path (red arrows).  For $d_i=[1, 1, 1]\t, i\in\mathbb{Z}_1^{10}$, it is observed in Fig.~\ref{distur_1_Lissajous_12}~(c)-(d) that ten robots from different initial positions maintain a platoon-like formation, but only deviate the desired 3D Lissajous path by a certain distance. However, for $d_i=[3,3,3]\t, i\in\mathbb{Z}_1^{10}$, Fig.~\ref{distur_1_Lissajous_12}~(e)-(f) exhibits that even the {\it spontaneous-ordering} platoon cannot be guaranteed anymore, which implies that the robustness of the present DGVF~\eqref{desired_law} holds for the disturbances with intensities smaller than a threshold $d_i=[3,3,3]\t, i\in\mathbb{Z}_1^{10}$.

Additionally, we take
Fig.~\ref{distur_1_Lissajous_12} (e) as an illustrative example to analyze the states' evolution of the {\it spontaneous-ordering} platoon under constant disturbances $d_i=[3, 3, 3]\t, i\in\mathbb{Z}_1^{10}$. As shown in Fig.~\ref{Dis_Lissajous_performance1}, the path-following errors $\phi_{i,1}, \phi_{i,2}, \phi_{i,3}, i\in \mathbb{Z}_1^{10}$ oscillate sharply and deviate from zeros, which implies that Claim 1 in Definition~\ref{CPF_definition}) cannot be guaranteed. Moreover, it is observed in Fig.~\ref{Dis_Lissajous_performance2} that the derivative of virtual coordinate $\dot{\omega}_i(t), i\in{\cal V}$ oscillates around~$-1$, which implies that $\lim_{t\rightarrow\infty} \dot{\omega}_{i}(t)=\dot{\omega}_{k}(t)\neq 0, \forall i\neq k \in{\cal V}$ in Claim 2) of Definition~\ref{CPF_definition} does not hold. 
Moreover,  Fig.~\ref{Dis_Lissajous_performance2} exhibits that the final relative values of adjacent virtual coordinates oscillate and cannot satisfy $0.4<\lim_{t\rightarrow\infty} |\omega_{s[k]}(t)-\omega_{s[k+1]}(t)|<0.6, \forall k\in\mathbb{Z}_1^{10}$ in the zoomed-in panels $[24\mbox{s},25\mbox{s}]\times [0.3, 0.7]$ compared with Fig.~\ref{self_intersected_path2}, i.e., Claim 3) of Definition~\ref{CPF_definition} does not hold. Therefore, it concludes that DGVF \eqref{desired_law} fails to guarantee {\it spontaneous-ordering} platoon under external disturbances with intensities greater  than the threshold $d_i=[3,3,3]\t, i\in\mathbb{Z}_1^{10}$.

\section{Conclusion}
In this paper, we have presented a DGVF algorithm such
that multiple robots are capable of forming a {\it spontaneous-ordering} platoon and moving along a predefined desired path in the $n$-dimensional Euclidean space. 
In particular, we add the path parameter as a virtual coordinate for each robot and then interact 
with neighboring robots' virtual coordinates and a target virtual coordinate. In this way, the robots are governed to approach the desired path and achieve a platoon in an arbitrary ordering. 
The conditions are derived to guarantee the global convergence of the proposed DGVF subject to time-varying interaction topologies and external exponentially vanishing disturbances.
Moreover, the DGVF algorithm only requires low communication costs by transmitting only virtual coordinates among robots, which is desirable in real applications. 
2D multi-USV waterway navigation experiments and 3D numerical simulations have shown the effectiveness and robustness of the 
proposed DGVF even if some robots break down and  suffer from small disturbances. Future work will focus on string stability analysis of the {\it spontaneous-ordering} platoon with time-varying neighbors.

\appendices
\section{Proof of Lemma~\ref{lemma_finiteescape}}
\label{Proof_lemma_1}
First of all, recalling $|\omega_{i,k}(t)|>r, \forall t\geq 0, \;  \forall i\neq k \in{\cal V}$ in Claim 4) of Definition~\ref{CPF_definition}, one has that $\omega_{i,k}(t)\neq0, \omega_{i,k}(t)\neq r, \forall t>0$
can be guaranteed if Claim 4) holds, i.e., the finite-time-escape behavior is avoided. Then, we will prove $|\omega_{i,k}(t)|>r, \forall t\geq 0, \;  \forall i\neq k \in{\cal V}$ by contradiction.

Let $\widetilde{\omega}_i:= \omega_i-\omega^{\ast},$ be the coordinate error between the 
$i$-th virtual coordinate~$\omega_i$ and the target virtual 
coordinate $\omega^{\ast}$, 
$F_i:=[\partial f_{i,1}, \dots, \partial f_{i,n}]\t\in \mathbb{R}^{n},$ and $K_i:=\mbox{diag}\{k_{i,1}, $ $\dots, k_{i,n}\}\in \mathbb{R}^{n\times n}$, and substitute Eq.~\eqref{desired_law} into Eq.~\eqref{dynamic_path} yields
\begin{align}
\label{dynamic_path_4}
\begin{bmatrix}
   \dot{\Phi}_i\\
   \dot{\widetilde{\omega}}_i\\
\end{bmatrix} 
=& 
\begin{bmatrix}                      
    -K_i(I_n+F_iF_i\t)\Phi_i \\
    F_i\t K_i\Phi_i
\end{bmatrix}+
\begin{bmatrix}   
     -F_i(-c_i\widetilde{\omega}_i+\eta_i+e_i)+\widetilde{d}_i  \\
    -c_i\widetilde{\omega}_i+\eta_i+e_i\\
\end{bmatrix},                                   
\end{align}
where 
$ \Phi_i$ is
given in \eqref{dynamic_path}, 
$I_n\in \mathbb{R}^{n\times n}$ is 
an identity matrix, $\widetilde{d}_i:=\widehat{d}_i-d_i$ are the estimated disturbance errors and 
$e_i:=c_i(\widehat{\omega}_i-\omega^{\ast})$.
Recalling Remark~\ref{observer_disturbance}, conditions C3 and C5, one has  
\begin{align}
\label{convergence_e}
\lim_{t\rightarrow T_1} \widetilde{d}_i(t)=0~\mbox{and}~\lim_{t\rightarrow\infty}e_i(t)=0,
\end{align}
exponentially.

Since condition C1 ensures that $|\omega_{i,k}(0)|>r, \; \forall i\neq k \in {\cal V}$ at the initial time, we assume that there exists a finite time $T>0$ such that $|\omega_{i,k}(t)|>r, \; \forall i\neq k \in {\cal V}$ for $t\in[0, T)$ but not $t=T$, which implies that at least one pair of virtual coordinates satisfies 
\begin{align}
\label{lemma_assump1}
\omega_{i,k}(T)\leq r.
\end{align} 
During the time interval $t\in[0, T)$, the closed-loop system~\eqref{dynamic_path_4} is well defined due to the fact that $|\omega_{i,k}(t)|>r, \; \forall i\neq k \in {\cal V}$. Then, we can 
pick a candidate Lyapunov function 
\begin{align}
\label{V_1}
V(t)=&\frac{1}{2}\sum\limits_{i\in\mathcal{ V}}\bigg\{\Phi_i\t K_i\Phi_i+c_i\widetilde{\omega}_i^2\bigg\}+\sum\limits_{i\in\mathcal{V}}\sum\limits_{k\in\mathcal N_i}\int_{|\omega_{i,k}|}^{R}\alpha(\tau)d\tau,
\end{align} 
which is nonnegative and differentiable in $t\in[0, T)$. The partial derivatives of $V(t)$ w.r.t. $\Phi_{i}, \omega^{\ast}, \omega_i$ 
are, respectively,
\begin{align}
\label{partial_diff}
\frac{\partial V(t)}{\partial \Phi_{i}\t}
=&
\Phi_{i}\t K_{i}, \frac{\partial V(t)}{\partial \omega^{\ast}}=-c_i\widetilde{\omega}_i, \nonumber\\
\frac{\partial V(t)}{\partial \omega_i}
=&
c_i\widetilde{\omega}_i-\sum\limits_{k\in\mathcal N_i}\alpha(|\omega_{i,k}|)\frac{\omega_{i,k}}{|\omega_{i,k}|}
=c_i\widetilde{\omega}_i-\eta_i,
\end{align}
it follows from Eqs.~\eqref{dynamic_path_4} and~\eqref{partial_diff} that the time derivative of~$V(t)$ is 
\begin{align}
\label{dot_V}
\frac{dV}{dt}
=&
\sum_{i\in{\cal V}}\bigg\{\frac{\partial V}{\partial \Phi_{i}\t}\dot{\Phi}_{i}
			+\frac{\partial V}{\partial \omega_i}\dot{\omega}_i+\frac{\partial V}{\partial \omega^{\ast}}\dot{\omega}^{\ast}\bigg\}\nonumber\\
=&
\sum_{i\in{\cal V}}\bigg\{\Phi_{i}\t K_{i}\big(-K_i(I_n+F_iF_i\t)
\Phi_i -F_i(-c_i\widetilde{\omega}_i \nonumber\\
&
+\eta_i+e_i)+\widetilde{d}_i\big)+(c_i\widetilde{\omega}_i-\eta_i)\dot{\omega}_i
-c_i\widetilde{\omega}_i\dot{\omega}^{\ast}\bigg\}.
\end{align}
From the fact $\dot{\widetilde{\omega}}_i=\dot{\omega}_i-\dot{\omega}^{\ast}$, 
one has
\begin{align}
\label{change_1}
c_i\widetilde{\omega}_i\dot{\omega}_i
-
c_i\widetilde{\omega}_i\dot{\omega}^{\ast}
=c_i\widetilde{\omega}_i\dot{\widetilde{\omega}}_i.
\end{align} 
Meanwhile, it follows from the definition of $\alpha(s)$ in \eqref{alpha} that 
$\sum_{i\in{\cal V}}(\eta_i\dot{\omega}^{\ast})
=
\dot{\omega}^{\ast}\sum_{i\in{\cal V}}\eta_i=0$, 
which implies that
\begin{align}
\label{change_2}
\sum_{i\in{\cal V}}\eta_i\dot{\omega}_i
=
\sum_{i\in{\cal V}}\eta_i(\dot{\omega}_i-\dot{\omega}^{\ast})
=
\sum_{i\in{\cal V}}\eta_i\dot{\widetilde{\omega}}_i.
\end{align} 
Combining Eqs.~\eqref{change_1} and \eqref{change_2} together yields
\begin{align}
\label{property1}
\sum_{i\in{\cal V}}\Big\{(c_i\widetilde{\omega}_i
-\eta_i)\dot{\omega}_i-c_i\widetilde{\omega}_i
\dot{\omega}^{\ast}\Big\}
=
\sum_{i\in{\cal V}}(c_i\widetilde{\omega}_i
-\eta_i)\dot{\widetilde{\omega}}_i.
\end{align}
Substituting Eq.~\eqref{property1} and $\dot{\widetilde{\omega}}_i$ 
in Eq.~\eqref{dynamic_path_4} into Eq.~\eqref{dot_V} yields
\begin{align}
\label{dot_V1}
\frac{dV}{dt}
=&
\sum_{i\in{\cal V}}\bigg\{-\Phi_{i}\t K_iK_i\Phi_{i}-\Phi_{i}\t K_{i}\t F_iF_i\t K_i\Phi_i\nonumber\\
&
-(-c_i\widetilde{\omega}_i+\eta_i)^2 -2\Phi_{i}\t K_iF_i(-c_i\widetilde{\omega}_i+\eta_i)\nonumber\\
&
-\Phi_{i}\t K_iF_ie_i+\Phi_{i}\t K_{i}\widetilde{d}_i-(-c_i\widetilde{\omega}_i+\eta_i)e_i\bigg\}.
\end{align}
From the definition of 
$F_i, \Phi_i, K_i$ 
in \eqref{dynamic_path_4}, one has that 
$F_i\t K_i\Phi_i=\Phi_{i}\t K_{i}\t F_i$
is a scalar, which implies that
\begin{align}
\label{equal_proprty}
&
-\Phi_{i}\t K_{i}\t F_iF_i\t K_i\Phi_i-(-c_i\widetilde{\omega}_i+\eta_i)^2\nonumber\\
&
-2\Phi_{i}\t K_iF_i(-c_i\widetilde{\omega}_i+\eta_i)-\Phi_{i}\t K_iF_ie_i-(-c_i\widetilde{\omega}_i+\eta_i)e_i\nonumber\\
=&
-(\Phi_{i}\t K_iF_i-c_i\widetilde{\omega}_i+\eta_i
+\frac{e_i}{2})^2+\frac{e_i^2}{4}.
\end{align}
Moreover, one has
\begin{align}
\label{inequal_proprty}
\Phi_{i}\t K_{i} \widetilde{d}_i\leq\frac{\Phi_{i}\t K_{i}K_i\Phi_{i}}{2}+\frac{ \widetilde{d}_i\t  \widetilde{d}_i}{2}.
\end{align}
Then, it follows from Eqs.~\eqref{dot_V1},~\eqref{equal_proprty} and \eqref{inequal_proprty} that 
\begin{align}
\label{dot_V2}
\frac{dV(t)}{dt}
=&
-\sum_{i\in{\cal V}}\bigg\{\frac{\Phi_{i}\t K_i 
K_i\Phi_{i}}{2}+a_i^2\bigg\}+
\sum_{i\in{\cal V}}\bigg\{{\frac{e_i^2}{4}+\frac{ \widetilde{d}_i\t  \widetilde{d}_i}{2}}\bigg\}
\end{align}
with 
\begin{align}
\label{replace_val}
a_i:=\Phi_{i}\t K_iF_i-c_i\widetilde{\omega}_i+\eta_i+\frac{e_i}{2}.
\end{align}
From the condition of $\lim_{t\rightarrow T_1} \widetilde{d}_i(t)=0$ and $ \lim_{t\rightarrow\infty}e_i(t)=0, i\in {\cal V},$ 
exponentially in~\eqref{convergence_e}, one has $\lim_{t\rightarrow\infty}\sum_{i\in{\cal V}}$ ${e_i(t)^2}/{4}=0, \lim_{t\rightarrow\infty}\sum_{i\in{\cal V}}{ \widetilde{d}_i(t)\t \widetilde{d}_i(t)}/{2}=0$, which implies that there exists a constant $\delta>0$ such that 
\begin{align}
\label{bound_convergence}
\sum_{i\in{\cal V}}\bigg\{\frac{e_i(t)^2}{4}+\frac{ \widetilde{d}_i(t)\t  \widetilde{d}_i(t)}{2}\bigg\}\leq \delta, \forall t\in[0, T).
\end{align}
Let
\begin{align}
\label{value_replace}
\Omega:
=&
-\sum_{i\in{\cal V}}\bigg\{\frac{\Phi_{i}\t K_iK_i\Phi_{i}}{2}+a_i^2\bigg\}\leq 0,
\end{align}
it follows from Eqs.~\eqref{dot_V2}, \eqref{bound_convergence}, \eqref{value_replace} that
\begin{align*}
\frac{dV(t)}{dt}
\leq&
-\Omega+\delta,
\end{align*}
which implies 
\begin{align}
\label{intel_V}
0
\leq 
V(T)
\leq
\int_{0}^T\Omega(s)ds+\delta T+V(0)
\end{align}
according to the comparison principle \cite{khalil2002nonlinear}.
From Eq.~\eqref{value_replace}, one has $\int_{0}^T\Omega(s)ds\leq 0$. 
Moreover, since $\delta T$ and $V(0)$ are both bounded, So is $V(T)$.

However, recalling the assumption of $\omega_{i,k}(T)\leq r$ in \eqref{lemma_assump1}, it follows from Eq.~\eqref{alpha} that $\alpha(|\omega_{i,k}(T)|)=\infty$, which further implies that 
$V(T)=\infty$. It contradicts the bounded value $V(T)$ in~\eqref{intel_V}, which indicates that there exists no such a finite $T$ satisfying $\omega_{i,k}(T)\leq r, \forall i\neq k\in{\cal V}$ (i.e., $T=\infty$). Then, we conclude  
$|\omega_{i,k}(t)|>r, \forall t\geq 0, \;  \forall i\neq k \in{\cal V}$. The proof of Claim 4) is thus completed.


\bibliographystyle{IEEEtran}
\bibliography{IEEEabrv,ref}

\begin{thebibliography}{10}
\providecommand{\url}[1]{#1}
\csname url@samestyle\endcsname
\providecommand{\newblock}{\relax}
\providecommand{\bibinfo}[2]{#2}
\providecommand{\BIBentrySTDinterwordspacing}{\spaceskip=0pt\relax}
\providecommand{\BIBentryALTinterwordstretchfactor}{4}
\providecommand{\BIBentryALTinterwordspacing}{\spaceskip=\fontdimen2\font plus
\BIBentryALTinterwordstretchfactor\fontdimen3\font minus
  \fontdimen4\font\relax}
\providecommand{\BIBforeignlanguage}[2]{{%
\expandafter\ifx\csname l@#1\endcsname\relax
\typeout{** WARNING: IEEEtran.bst: No hyphenation pattern has been}%
\typeout{** loaded for the language `#1'. Using the pattern for}%
\typeout{** the default language instead.}%
\else
\language=\csname l@#1\endcsname
\fi
#2}}
\providecommand{\BIBdecl}{\relax}
\BIBdecl

\bibitem{macwan2014multirobot}
A.~Macwan, J.~Vilela, G.~Nejat, and B.~Benhabib, ``A multirobot path-planning
  strategy for autonomous wilderness search and rescue,'' \emph{IEEE
  Transactions on Cybernetics}, vol.~45, no.~9, pp. 1784--1797, 2014.

\bibitem{dunbabin2012robots}
M.~Dunbabin and L.~Marques, ``Robots for environmental monitoring: Significant
  advancements and applications,'' \emph{IEEE Robotics \& Automation Magazine},
  vol.~19, no.~1, pp. 24--39, 2012.

\bibitem{hu2021distributed2}
B.-B. Hu, Z.~Chen, and H.-T. Zhang, ``Distributed moving target fencing in a
  regular polygon formation,'' \emph{IEEE Transactions on Control of Network
  Systems}, vol.~9, no.~1, pp. 210--218, 2022.

\bibitem{alonso2015multi}
J.~Alonso-Mora, S.~Baker, and D.~Rus, ``Multi-robot navigation in formation via
  sequential convex programming,'' in \emph{Proceeding of International
  Conference on Intelligent Robots and Systems (IROS)}, 2015, pp. 4634--4641.

\bibitem{hu2023cooperative}
B.-B. Hu, H.-T. Zhang, and Y.~Shi, ``Cooperative label-free moving target
  fencing for second-order multi-agent systems with rigid formation,''
  \emph{Automatica}, vol. 148, p. 110788, 2023.

\bibitem{samson1995control}
C.~Samson, ``Control of chained systems application to path following and
  time-varying point-stabilization of mobile robots,'' \emph{IEEE Transactions
  on Automatic Control}, vol.~40, no.~1, pp. 64--77, 1995.

\bibitem{aguiar2007trajectory}
A.~P. Aguiar and J.~P. Hespanha, ``Trajectory-tracking and path-following of
  underactuated autonomous vehicles with parametric modeling uncertainty,''
  \emph{IEEE Transactions on Automatic Control}, vol.~52, no.~8, pp.
  1362--1379, 2007.

\bibitem{fossen2003line}
T.~I. Fossen, M.~Breivik, and R.~Skjetne, ``Line-of-sight path following of
  underactuated marine craft,'' \emph{IFAC Proceedings Volumes}, vol.~36,
  no.~21, pp. 211--216, 2003.

\bibitem{rysdyk2006unmanned}
R.~Rysdyk, ``Unmanned aerial vehicle path following for target observation in
  wind,'' \emph{Journal of Guidance, Control, and Dynamics}, vol.~29, no.~5,
  pp. 1092--1100, 2006.

\bibitem{kapitanyuk2017guiding}
Y.~A. Kapitanyuk, A.~V. Proskurnikov, and M.~Cao, ``A guiding vector-field
  algorithm for path-following control of nonholonomic mobile robots,''
  \emph{IEEE Transactions on Control Systems Technology}, vol.~26, no.~4, pp.
  1372--1385, 2017.

\bibitem{yao2020path}
W.~Yao and M.~Cao, ``Path following control in 3{D} using a vector field,''
  \emph{Automatica}, vol. 117, p. 108957, 2020.

\bibitem{yao2019distributed}
W.~Yao, H.~Lu, Z.~Zeng, J.~Xiao, and Z.~Zheng, ``Distributed static and dynamic
  circumnavigation control with arbitrary spacings for a heterogeneous
  multi-robot system,'' \emph{Journal of Intelligent \& Robotic Systems},
  vol.~94, no.~3, pp. 883--905, 2019.

\bibitem{burger2009straight}
M.~Burger, A.~Pavlov, E.~Borhaug, and K.~Y. Pettersen, ``Straight line path
  following for formations of underactuated surface vessels under influence of
  constant ocean currents,'' in \emph{Proceeding of American Control Conference
  (ACC)}, 2009, pp. 3065--3070.

\bibitem{ghommam2010formation}
J.~Ghommam, H.~Mehrjerdi, M.~Saad, and F.~Mnif, ``Formation path following
  control of unicycle-type mobile robots,'' \emph{Robotics and Autonomous
  Systems}, vol.~58, no.~5, pp. 727--736, 2010.

\bibitem{liu2020scanning}
B.~Liu, H.-T. Zhang, H.~Meng, D.~Fu, and H.~Su, ``Scanning-chain formation
  control for multiple unmanned surface vessels to pass through water
  channels,'' \emph{IEEE Transactions on Cybernetics}, vol.~52, no.~3, pp.
  1850--1861, 2020.

\bibitem{hu2021distributed1}
B.-B. Hu, H.-T. Zhang, B.~Liu, H.~Meng, and G.~Chen, ``Distributed surrounding
  control of multiple unmanned surface vessels with varying interconnection
  topologies,'' \emph{IEEE Transactions on Control Systems Technology},
  vol.~30, no.~1, pp. 400--407, 2021.

\bibitem{hu2021bearing}
B.-B. Hu and H.-T. Zhang, ``Bearing-only motional target-surrounding control
  for multiple unmanned surface vessels,'' \emph{IEEE Transactions on
  Industrial Electronics}, vol.~69, no.~4, pp. 3988--3997, 2021.

\bibitem{zhang2007coordinated}
F.~Zhang and N.~E. Leonard, ``Coordinated patterns of unit speed particles on a
  closed curve,'' \emph{Systems \& Control Letters}, vol.~56, no.~6, pp.
  397--407, 2007.

\bibitem{nakai2013vector}
K.~Nakai and K.~Uchiyama, ``Vector fields for uav guidance using potential
  function method for formation flight,'' in \emph{AIAA Guidance, Navigation,
  and Control (GNC) Conference}, 2013, p. 4626.

\bibitem{de2017circular}
H.~G. De~Marina, Z.~Sun, M.~Bronz, and G.~Hattenberger, ``Circular formation
  control of fixed-wing {UAV}s with constant speeds,'' in \emph{Proceeding of
  International Conference on Intelligent Robots and Systems (IROS)}, 2017, pp.
  5298--5303.

\bibitem{doosthoseini2015coordinated}
A.~Doosthoseini and C.~Nielsen, ``Coordinated path following for unicycles: A
  nested invariant sets approach,'' \emph{Automatica}, vol.~60, pp. 17--29,
  2015.

\bibitem{sabattini2015implementation}
L.~Sabattini, C.~Secchi, M.~Cocetti, A.~Levratti, and C.~Fantuzzi,
  ``Implementation of coordinated complex dynamic behaviors in multirobot
  systems,'' \emph{IEEE Transactions on Robotics}, vol.~31, no.~4, pp.
  1018--1032, 2015.

\bibitem{pimenta2013decentralized}
L.~C. Pimenta, G.~A. Pereira, M.~M. Gon{\c{c}}alves, N.~Michael, M.~Turpin, and
  V.~Kumar, ``Decentralized controllers for perimeter surveillance with teams
  of aerial robots,'' \emph{Advanced Robotics}, vol.~27, no.~9, pp. 697--709,
  2013.

\bibitem{yao2021distributed}
W.~Yao, H.~G. de~Marina, Z.~Sun, and M.~Cao, ``Distributed coordinated path
  following using guiding vector fields,'' in \emph{Proceeding of IEEE
  International Conference on Robotics and Automation (ICRA)}, 2021, pp.
  10\,030--10\,037.

\bibitem{yao2022guiding}
------, ``Guiding vector fields for the distributed motion coordination of
  mobile robots,'' \emph{IEEE Transactions on Robotics}, in press, doi:
  10.1109/TRO.2022.3224257, 2022.

\bibitem{sakurama2020multi}
K.~Sakurama and H.-S. Ahn, ``Multi-agent coordination over local indexes via
  clique-based distributed assignment,'' \emph{Automatica}, vol. 112, p.
  108670, 2020.

\bibitem{lan2011synthesis}
Y.~Lan, G.~Yan, and Z.~Lin, ``Synthesis of distributed control of coordinated
  path following based on hybrid approach,'' \emph{IEEE Transactions on
  Automatic Control}, vol.~56, no.~5, pp. 1170--1175, 2011.

\bibitem{reyes2014flocking}
L.~A.~V. Reyes and H.~G. Tanner, ``Flocking, formation control, and path
  following for a group of mobile robots,'' \emph{IEEE Transactions on Control
  Systems Technology}, vol.~23, no.~4, pp. 1268--1282, 2014.

\bibitem{ploeg2013controller}
J.~Ploeg, D.~P. Shukla, N.~van~de Wouw, and H.~Nijmeijer, ``Controller
  synthesis for string stability of vehicle platoons,'' \emph{IEEE Transactions
  on Intelligent Transportation Systems}, vol.~15, no.~2, pp. 854--865, 2013.

\bibitem{ploeg2013lp}
J.~Ploeg, N.~Van De~Wouw, and H.~Nijmeijer, ``Lp string stability of cascaded
  systems: Application to vehicle platooning,'' \emph{IEEE Transactions on
  Control Systems Technology}, vol.~22, no.~2, pp. 786--793, 2013.

\bibitem{besselink2017string}
B.~Besselink and K.~H. Johansson, ``String stability and a delay-based spacing
  policy for vehicle platoons subject to disturbances,'' \emph{IEEE
  Transactions on Automatic Control}, vol.~62, no.~9, pp. 4376--4391, 2017.

\bibitem{dunbar2011distributed}
W.~B. Dunbar and D.~S. Caveney, ``Distributed receding horizon control of
  vehicle platoons: Stability and string stability,'' \emph{IEEE Transactions
  on Automatic Control}, vol.~57, no.~3, pp. 620--633, 2011.

\bibitem{monteil2019string}
J.~Monteil, G.~Russo, and R.~Shorten, ``On ${L}_{\infty}$ string stability of
  nonlinear bidirectional asymmetric heterogeneous platoon systems,''
  \emph{Automatica}, vol. 105, pp. 198--205, 2019.

\bibitem{hu2020cooperative}
J.~Hu, P.~Bhowmick, F.~Arvin, A.~Lanzon, and B.~Lennox, ``Cooperative control
  of heterogeneous connected vehicle platoons: An adaptive leader-following
  approach,'' \emph{IEEE Robotics and Automation Letters}, vol.~5, no.~2, pp.
  977--984, 2020.

\bibitem{mokogwu2022energy}
C.~N. Mokogwu and K.~Hashtrudi-Zaad, ``Energy-based analysis of string
  stability in vehicle platoons,'' \emph{IEEE Transactions on Vehicular
  Technology}, vol.~71, no.~6, pp. 5915--5929, 2022.

\bibitem{xu2022stochastic}
L.~Xu, X.~Jin, Y.~Wang, Y.~Liu, W.~Zhuang, and G.~Yin, ``Stochastic stable
  control of vehicular platoon time-delay system subject to random switching
  topologies and disturbances,'' \emph{IEEE Transactions on Vehicular
  Technology}, vol.~71, no.~6, pp. 5755--5769, 2022.

\bibitem{feng2022robust}
F.~Gao, D.~Dang, and Y.~He, ``Robust coordinated control of nonlinear
  heterogeneous platoon interacted by uncertain topology,'' \emph{IEEE
  Transactions on Intelligent Transportation Systems}, vol.~23, no.~6, pp.
  4982--4992.

\bibitem{liu2020internal}
Y.~Liu, H.~Gao, C.~Zhai, and W.~Xie, ``Internal stability and string stability
  of connected vehicle systems with time delays,'' \emph{IEEE Transactions on
  Intelligent Transportation Systems}, vol.~22, no.~10, pp. 6162--6174, 2020.

\bibitem{qin2019experimental}
W.~B. Qin and G.~Orosz, ``Experimental validation of string stability for
  connected vehicles subject to information delay,'' \emph{IEEE Transactions on
  Control Systems Technology}, vol.~28, no.~4, pp. 1203--1217, 2019.

\bibitem{seron1999feedback}
M.~M. Seron, J.~H. Braslavsky, P.~V. Kokotovic, and D.~Q. Mayne, ``Feedback
  limitations in nonlinear systems: From bode integrals to cheap control,''
  \emph{IEEE Transactions on Automatic Control}, vol.~44, no.~4, pp. 829--833,
  1999.

\bibitem{yao2020vector}
W.~Yao, H.~G. de~Marina, and M.~Cao, ``Vector field guided path following
  control: Singularity elimination and global convergence,'' in
  \emph{Proceeding of IEEE Conference on Decision and Control (CDC)}, 2020, pp.
  1543--1549.

\bibitem{el2007passivity}
M.~I. El-Hawwary and M.~Maggiore, ``Passivity-based stabilization of
  non-compact sets,'' in \emph{Proceeding of IEEE Conference on Decision and
  Control (CDC)}, 2007, pp. 1734--1739.

\bibitem{yao2018robotic}
W.~Yao, Y.~A. Kapitanyuk, and M.~Cao, ``Robotic path following in 3{D} using a
  guiding vector field,'' in \emph{Proceeding of IEEE Conference on Decision
  and Control (CDC)}, 2018, pp. 4475--4480.

\bibitem{rezende2018robust}
A.~M. Rezende, V.~M. Gon{\c{c}}alves, G.~V. Raffo, and L.~C. Pimenta, ``Robust
  fixed-wing {UAV} guidance with circulating artificial vector fields,'' in
  \emph{Proceeding of International Conference on Intelligent Robots and
  Systems (IROS)}, 2018, pp. 5892--5899.

\bibitem{goncalves2010vector}
V.~M. Goncalves, L.~C. Pimenta, C.~A. Maia, B.~C. Dutra, and G.~A. Pereira,
  ``Vector fields for robot navigation along time-varying curves in $ n
  $-dimensions,'' \emph{IEEE Transactions on Robotics}, vol.~26, no.~4, pp.
  647--659, 2010.

\bibitem{yao2021singularity}
W.~Yao, H.~G. de~Marina, B.~Lin, and M.~Cao, ``Singularity-free guiding vector
  field for robot navigation,'' \emph{IEEE Transactions on Robotics}, vol.~37,
  no.~4, pp. 1206--1221, 2021.

\bibitem{galbis2012vector}
A.~Galbis and M.~Maestre, \emph{Vector analysis versus vector calculus}.\hskip
  1em plus 0.5em minus 0.4em\relax Springer Science \& Business Media, 2012.

\bibitem{chen2019cooperative}
Z.~Chen, ``A cooperative target-fencing protocol of multiple vehicles,''
  \emph{Automatica}, vol. 107, pp. 591--594, 2019.

\bibitem{peng2020output}
Z.~Peng, L.~Liu, and J.~Wang, ``Output-feedback flocking control of multiple
  autonomous surface vehicles based on data-driven adaptive extended state
  observers,'' \emph{IEEE Transactions on Cybernetics}, vol.~51, no.~9, pp.
  4611--4622, 2020.

\bibitem{olfati2004consensus}
R.~Olfati-Saber and R.~M. Murray, ``Consensus problems in networks of agents
  with switching topology and time-delays,'' \emph{IEEE Transactions on
  Automatic Control}, vol.~49, no.~9, pp. 1520--1533, 2004.

\bibitem{hong2006tracking}
Y.~Hong, J.~Hu, and L.~Gao, ``Tracking control for multi-agent consensus with
  an active leader and variable topology,'' \emph{Automatica}, vol.~42, no.~7,
  pp. 1177--1182, 2006.

\bibitem{zhao2013distributed}
Y.~Zhao, Z.~Duan, G.~Wen, and Y.~Zhang, ``Distributed finite-time tracking
  control for multi-agent systems: An observer-based approach,'' \emph{Systems
  \& Control Letters}, vol.~62, no.~1, pp. 22--28, 2013.

\bibitem{gu2022disturbance}
N.~Gu, D.~Wang, Z.~Peng, J.~Wang, and Q.-L. Han, ``Disturbance observers and
  extended state observers for marine vehicles: A survey,'' \emph{Control
  Engineering Practice}, vol. 123, p. 105158, 2022.

\bibitem{khalil2002nonlinear}
H.~K. Khalil, \emph{Nonlinear Systems}.\hskip 1em plus 0.5em minus 0.4em\relax
  Upper Saddle River, 2002.

\bibitem{dong2021anti1}
R.-Q. Dong, A.-G. Wu, and Y.~Zhang, ``Anti-unwinding sliding mode attitude
  maneuver control for rigid spacecraft,'' \emph{IEEE Transactions on Automatic
  Control}, vol.~67, no.~2, pp. 978--985, 2021.

\bibitem{dong2021anti2}
R.-Q. Dong, A.-G. Wu, Y.~Zhang, and G.-R. Duan, ``Anti-unwinding sliding mode
  attitude control via two modified rodrigues parameter sets for spacecraft,''
  \emph{Automatica}, vol. 129, p. 109642, 2021.

\bibitem{Liubin2018surounding}
B.~Liu, Z.~Chen, H.~Zhang, X.~Wang, T.~Geng, H.~Su, and J.~Zhao, ``Collective
  dynamics and control for multiple unmanned surface vessels,'' \emph{IEEE
  Transactions on Control Systems Technology}, vol.~28, no.~6, pp. 2540--2547,
  2020.

\end{thebibliography}

\end{document}